\documentclass[twoside,11pt]{article}

\usepackage{amsmath,amsfonts,amstext}
\usepackage{latexsym}
\usepackage{subfigure}
\usepackage{color,graphicx}
\usepackage{algorithm,algorithmic}
\usepackage{natbib,amsthm,fullpage}

\newtheorem{definition}{Definition}
\newtheorem{theorem}{Theorem}
\newtheorem{corollary}{Corollary}
\newtheorem{lemma}{Lemma}

\newcommand{\norm}[1]{\left\| #1 \right\|}

\def\prob{\mathbb{P}}
\def\expe{\mathbb{E}}
\def\Vol{\mathbf{Vol}}
\def\mbf{\mathbf}
\def\mc{\mathcal}
\def\SS{\mathcal{S}}
\def\M{\mathcal{M}}
\def\argmin{\mathop{\rm argmin}}

\def\Vol{\mathop{\rm Vol}}

\def\bestloss{L^{\ast}}
\def\loss{\ell}
\def\exploss{L}
\def\obj{J}
\def\expobj{\bar{J}}
\def\reg{\Lambda}

\def\priveps{\epsilon_p}
\def\geneps{\epsilon_g}
\def\adult{\texttt{Adult}}
\def\kddcup{\texttt{KDDCup99}}
\newcommand{\card}[1]{| #1 |}
\def\bbR{\mathbb{R}}

\def\loss{\ell}
\def\lhuber{\loss_{\mathrm{Huber}}}
\def\lsvm{\loss_{\mathrm{SVM}}}
\def\llr{\loss_{\mathrm{LR}}}
\def\exploss{L}
\def\obj{J}
\def\expobj{\bar{J}}

\def\reg{\Lambda}

\def\surf{\textrm{surf}}
\def\priveps{\epsilon_p}
\def\geneps{\epsilon_g}
\def\extra{\Delta}
\def\c{c}

\def\D{\mathcal{D}}
\def\x{\mathbf{x}}

\def\f{\mathbf{f}}
\def\b{\mathbf{b}}
\def\fpriv{\mathbf{f}_{\mathrm{priv}}}
\def\fopt{\mathbf{f}_0}
\def\aopt{a_0}
\def\frtr{\mbf{f}_{\mathrm{rtr}}}

\def\femp{\mbf{f}_{\mathrm{rtr}}}

\newcommand{\Hnorm}[1]{ \norm{ #1 }_{\mc{H}} }
\def\regrtr{\Lambda_{\mathrm{rtr}}}
\def\kft{\bar{K}}

\def\f{\mathbf{f}}
\def\b{\mathbf{b}}
\def\D{\mathcal{D}}

\newcommand{\ignore}[1]{}

\newcommand{\kc}[1]{}
\newcommand{\ads}[1]{}
\newcommand{\cm}[1]{}

\title{Differentially Private Empirical Risk Minimization\thanks{This work will appear in the Journal of Machine Learning Research.}}

\author{Kamalika Chaudhuri\thanks{ K. Chaudhuri is with the Department of Computer Science and Engineering, University of California, San Diego, La Jolla, CA 92093, USA, \texttt{kchaudhuri@ucsd.edu}}, 
Claire Monteleoni\thanks{C. Monteleoni is with the Center for Computational Learning Systems, Columbia University, New York, NY 10115, USA, \texttt{cmontel@ccls.columbia.edu}},
Anand D. Sarwate\thanks{A.D. Sarwate is with the Information Theory and Applications Center, University of California, San Diego, La Jolla, CA 92093-0447, USA, \texttt{asarwate@ucsd.edu}}}

\date{\today}

\begin{document}

\maketitle

\begin{abstract}%
Privacy-preserving machine learning algorithms are crucial for the
increasingly common setting in which personal data, such as medical or
financial records, are analyzed.  We provide general techniques to
produce privacy-preserving approximations of classifiers learned via
(regularized) empirical risk minimization (ERM).  These algorithms are
private under the {\it $\epsilon$-differential privacy} definition due
to Dwork et al.~(2006).  First we apply the output perturbation ideas
of Dwork et al.~(2006), to ERM classification.  Then we propose a new
method, {\it objective perturbation}, for privacy-preserving machine
learning algorithm design.  This method entails perturbing the
objective function before optimizing over classifiers.  If the loss and
regularizer satisfy certain convexity and differentiability criteria, 
we prove theoretical results showing that our algorithms preserve privacy, and
provide generalization bounds for linear and nonlinear kernels.  We
further present a privacy-preserving technique for tuning the
parameters in general machine learning algorithms, thereby providing
end-to-end privacy guarantees for the training process. We apply these
results to produce privacy-preserving analogues of regularized
logistic regression and support vector machines.  We obtain
encouraging results from evaluating their performance on real
demographic and benchmark data sets.  Our results show that both
theoretically and empirically, objective perturbation is superior to
the previous state-of-the-art, output perturbation, in managing the
inherent tradeoff between privacy and learning performance.\end{abstract}

\section{Introduction \label{sec:relwork}}

Privacy has become a growing concern, due to the massive increase in personal information stored in electronic databases, such as medical records, financial records, web search histories, and social network data.  Machine learning can be employed to discover novel population-wide patterns, however the results of such algorithms may reveal certain individuals' sensitive information, thereby violating their privacy. Thus, an emerging challenge for machine learning is how to learn from datasets that contain sensitive personal information. 

At the first glance, it may appear that simple anonymization of private information is enough to preserve privacy. However, this is often not the case; even if obvious identifiers, such as names and addresses, are removed from the data, the remaining fields can still form unique ``signatures'' that can help re-identify individuals. Such attacks have been demonstrated by  various works, and are possible in many realistic settings, such as when an adversary has side information \citep{Sweeney:97weaving,netflix,GKS08}, and when the data has structural properties \citep{livejournal}, among others.  Moreover, even releasing statistics on a sensitive dataset may not be sufficient to preserve privacy, as illustrated on genetic data \citep{HomerEtAl:08plos,Wang10}. Thus, there is a great need for designing machine learning algorithms that also preserve the privacy of individuals in the datasets on which they train and operate.
 
In this paper we focus on the problem of classification, one of the fundamental problems of machine learning, when the training data consists of sensitive information of individuals. Our work addresses the Empirical risk minimization (ERM) framework for classification, in which a classifier is chosen by minimizing the average over the training data of the prediction loss (with respect to the label) of the classifier in predicting each training data point.  In this work, we focus on regularized ERM in which there is an additional term in the optimization, called the regularizer, which penalizes the complexity of the classifier with respect to some metric.  Regularized ERM methods are widely used in practice, for example in logistic regression and support vector machines (SVMs), and many also have theoretical justification in the form of generalization error bounds  with respect to independently, identically distributed (i.i.d.)~data (see \cite{vapnik} for further details).

For our privacy measure, we use a definition due to \cite{DworkMNS:06sensitivity}, who have proposed a measure of quantifying the privacy-risk associated with computing functions of sensitive data.  Their \textit{$\epsilon$-differential privacy} model is a strong, cryptographically-motivated definition of privacy that has recently received a significant amount of research attention for its robustness to known attacks, such as those involving side information \citep{GKS08}.  Algorithms satisfying $\epsilon$-differential privacy are randomized; the output is a random variable whose distribution is conditioned on the data set.  A statistical procedure satisfies $\epsilon$-differential privacy if changing a single data point does not shift the output distribution by too much.  Therefore, from looking at the output of the algorithm, it is difficult to infer the value of any particular data point.

In this paper, we develop methods for approximating ERM while guaranteeing $\epsilon$-differential privacy.   Our results hold for loss functions and regularizers satisfying certain differentiability and convexity conditions.  An important aspect of our work is that we develop methods for \textit{end-to-end privacy}; each step in the learning process can cause additional risk of privacy violation, and we provide algorithms with quantifiable privacy guarantees for training as well as parameter tuning.  For training, we provide two privacy-preserving approximations to ERM.  The first is \textit{output perturbation}, based on the \textit{sensitivity method} proposed by \cite{DworkMNS:06sensitivity}.  In this method noise is added to the output of the standard ERM algorithm.  The second method is novel, and involves adding noise to the regularized ERM objective function prior to minimizing.  We call this second method \textit{objective perturbation}.  We show theoretical bounds for both procedures; the theoretical performance of objective perturbation is superior to that of output perturbation for most problems.   However, for our results to hold we require that the regularizer be strongly convex (ruling $L_1$ regularizers) and additional constraints on the loss function and its derivatives.  In practice, these additional constraints do not affect the performance of the resulting classifier; we validate our theoretical results on data sets from the UCI repository.

In practice, parameters in learning algorithms are chosen via a holdout data set.  In the context of privacy, we must guarantee the privacy of the holdout data as well.   We exploit results from the theory of differential privacy to develop a privacy-preserving parameter tuning algorithm, and demonstrate its use in practice.  Together with our training algorithms, this parameter tuning algorithm guarantees privacy to all data used in the learning process.

Guaranteeing privacy incurs a cost in performance; because the algorithms
must cause some uncertainty in the output, they increase the loss of the
output predictor.  Because the $\epsilon$-differential privacy model
requires robustness against all data sets, we make no assumptions on the
underlying data for the purposes of making privacy guarantees.  However, to
prove the impact of privacy constraints on the generalization error, we
assume the data is i.i.d.~according to a fixed but unknown distribution, as is standard in the machine learning literature.
Although many of our results hold for ERM in general, we provide specific
results for classification using logistic regression and support vector
machines.  Some of the former results were reported in
\cite{ChaudhuriM:08nips}; here we generalize them to ERM and extend the
results to kernel methods, and provide experiments on real datasets.  

More specifically, the contributions of this paper are as follows:
	\begin{itemize}
	\item We derive a computationally efficient algorithm for ERM classification, based on the
	sensitivity method due to \cite{DworkMNS:06sensitivity}. We analyze
	the accuracy of this algorithm, and provide an upper bound on the
	number of training samples required by this algorithm to achieve
	a fixed generalization error.
	\item We provide a general technique, {\it objective perturbation},
	for providing computationally efficient, differentially private approximations to
	regularized ERM algorithms.  This extends the work of
	\cite{ChaudhuriM:08nips}, which follows as a special case, and corrects an error in the arguments made there.  We
	apply the general results on the sensitivity method and objective
	perturbation to logistic regression and support vector machine
	classifiers. In addition to privacy guarantees, we also provide
	generalization bounds for this algorithm.
    \item For kernel methods with nonlinear kernel functions, the optimal classifier is a linear combination of kernel functions centered at the training points.  This form is inherently non-private because it reveals the training data.  We adapt a random projection method due to Rahimi and Recht \citep{RahimiR:07features,RahimiR:08kitchen}, to develop privacy-preserving kernel-ERM algorithms.  We provide theoretical results on generalization performance.
	\item Because the holdout data is used in the process of training and releasing a classifier, we provide a privacy-preserving parameter tuning algorithm based on a randomized selection procedure \citep{MT07} applicable to general machine learning algorithms.  This guarantees end-to-end privacy during the learning procedure.
	\item We validate our results using experiments on two datasets
	from the UCI %
Machine
	Learning repositories \citep{uciadult}) and KDDCup \citep{kddcup99}.  Our results show that objective perturbation is generally superior to output perturbation.  We also demonstrate the impact of end-to-end privacy on generalization error.
	\end{itemize}

\subsection{Related Work}  

There has been a significant amount of literature on the ineffectiveness of
simple anonymization procedures. For example, \cite{netflix} show that a
small amount of auxiliary information (knowledge of a few movie-ratings,
and approximate dates) is sufficient for an adversary to re-identify an
individual in the Netflix dataset, which consists of anonymized data about
Netflix users and their movie ratings. The same phenomenon has been
observed in other kinds of data, such as social network graphs
\citep{livejournal}, search query logs \citep{JKPT07} and others. 
Releasing statistics computed on sensitive data can also be problematic;
for example, \cite{Wang10} show that releasing $R^2$-values computed on high-dimensional genetic data can lead to privacy breaches by an adversary who is armed with a small amount of auxiliary information.

There has also been a significant amount of work on privacy-preserving data mining~\citep{AS00,EGS03,S02,kifer-diversity}, spanning several communities, that uses privacy models other than differential privacy.  Many of the models used have been shown to be susceptible to {\it composition attacks}, attacks in which the adversary has some reasonable amount of prior knowledge~\citep{GKS08}.  Other work \citep{MWF08} considers the problem of privacy-preserving SVM classification when separate agents have to share private data, and provides a solution that uses random kernels, but does provide any formal privacy guarantee. 

An alternative line of privacy work is in the Secure Multiparty Computation
setting due to \cite{Yao82}, where the sensitive data is split across
multiple hostile databases, and the goal is to compute a function on the
union of these databases.  \cite{ZM07} and \cite{LLM06} consider computing
privacy-preserving SVMs in this setting, and their goal is to design a
distributed protocol to learn a classifier.  This is in contrast with our
work, which deals with a setting where the algorithm has access to the entire dataset.

Differential privacy, the formal privacy definition used in our paper, was
proposed by the seminal work of \cite{DworkMNS:06sensitivity},
and has been used since in numerous works on privacy \citep{CM06, MT07,NRS07,BCDKMT07,ChaudhuriM:08nips,MKAGV08}. Unlike many
other privacy definitions, such as those mentioned above, differential
privacy has been shown to be resistant to composition attacks (attacks
involving side-information) \citep{GKS08}. Some follow-up work on
differential privacy includes work on differentially-private combinatorial
optimization, due to~\cite{GLMRT10}, and differentially-private contingency
tables, due to ~\cite{BCDKMT07} and~\cite{KRS10}.
\cite{WassermanZ:10framework} provide a more statistical view of
differential privacy, and \cite{ZhouLW:09compression} provide a technique of
generating synthetic data using compression via random linear or affine
transformations.

Previous literature has also considered learning with differential privacy.
One of the first such works is \cite{KLNRS08}, which presents a
general, although computationally inefficient, method for PAC-learning
finite concept classes.  \cite{BLR08} presents a method for releasing a
database in a differentially-private manner, so that certain fixed classes
of queries can be answered accurately, provided the class of queries has a
bounded VC-dimension. Their methods can also be used to learn
classifiers with a fixed VC-dimension -- see \cite{KLNRS08}; however the resulting algorithm is also computationally inefficient. Some sample complexity lower bounds in
this setting have been provided by~\cite{BKN10}. In addition, \cite{DL09}
explore a connection between differential privacy and robust statistics,
and provide an algorithm for privacy-preserving regression using ideas from
robust statistics.  However, their algorithm also requires a running time
which is exponential in the data dimension, and is hence computationally
inefficient.

This work builds on our preliminary work in \cite{ChaudhuriM:08nips}.  We first show how to extend the sensitivity method, a form of {\it output perturbation}, due to \cite{DworkMNS:06sensitivity}, to classification algorithms.   In general, output perturbation methods alter the output of the function computed on the database, before releasing it; in particular the sensitivity method makes an algorithm differentially private by adding noise to its output.  In the classification setting, the noise protects the privacy of the training data, but increases the prediction error of the classifier.  Recently, independent work by \cite{rub} has reported an extension of the sensitivity method to linear and kernel SVMs. Their utility analysis differs from ours, and thus the analogous generalization bounds are not comparable.  Because Rubinstein et al. use techniques from algorithmic stability, their utility bounds compare the private and non-private classifiers using the same value for the regularization parameter.  In contrast, our approach takes into account how the value of the regularization parameter might change due to privacy constraints.  In contrast, we propose the {\it objective perturbation} method, in which noise is added to the {\it objective function} before optimizing over the space classifiers. Both the sensitivity method and objective perturbation result in computationally efficient algorithms for our specific case.  In general, our theoretical bounds on sample requirement are incomparable with the bounds of~\cite{KLNRS08} because of the difference between their setting and ours.

Our approach to privacy-preserving tuning uses the exponential mechanism of
\cite{MT07} by training classifiers with different parameters on disjoint subsets of the data and then randomizing the selection of which classifier to release.  This bears a superficial resemblance to the sample-and-aggregate \citep{NRS07} and V-fold cross-validation, but only in the sense that only a part of the data is used to train the classifier.   One drawback is that our approach requires significantly more data in practice.  Other approaches to selecting the regularization parameter could benefit from a more careful analysis of the regularization parameter, as in \cite{HastieRTZ:04reg}.

\section{Model}

We will use $\norm{\mbf{x}}$, $\norm{\mbf{x}}_{\infty}$, and $\norm{\mbf{x}}_{\mc{H}}$ to denote the $\ell_2$-norm, $\ell_{\infty}$-norm, and norm in a Hilbert space $\mc{H}$, respectively.  For an integer $n$ we will use $[n]$ to denote the set $\{1,2, \ldots, n\}$.  Vectors will typically be written in boldface and sets in calligraphic type. For a matrix $A$, we will use the notation $\norm{A}_2$ to denote the $L_2$ norm
of $A$.

\subsection{Empirical Risk Minimization}

In this paper we develop privacy-preserving algorithms for
\textit{regularized empirical risk minimization}, a special case of which
is learning a classifier from labeled examples.  We will phrase our problem
in terms of classification and indicate when more general results hold.
Our algorithms take as input \textit{training data} $\mc{D} = \{
(\mbf{x}_i, \mbf{y}_i) \in \mc{X} \times \mc{Y} : i = 1, 2, \ldots, n\}$
of $n$ data-label pairs.  In the case of binary classification the data
space $\mc{X} = \mathbb{R}^d$ and the label set $\mc{Y} = \{-1,+1\}$.  We
will assume throughout that $\mc{X}$ is the unit ball so that
$\norm{\mbf{x}_i}_2 \le 1$.

We would like to produce a \textit{predictor} $\mbf{f} : \mc{X} \to \mc{Y}$.  We measure the quality of our predictor on the training data via a nonnegative \textit{loss function} $\loss : \mc{Y} \times \mc{Y} \to \mathbb{R}$. %

In regularized empirical risk minimization (ERM), we choose a predictor $\mbf{f}$ that minimizes the regularized empirical loss:
	\begin{align}
	\obj(\mbf{f},\mc{D}) = \frac{1}{n} \sum_{i=1}^{n} \loss(
	\mbf{f}(\mbf{x}_i),\mbf{y}_i) + \reg N(\mbf{f}).
	\label{eq:StdRegERM}
	\end{align}
This minimization is performed over $\mbf{f}$ in an hypothesis class $\mc{H}$.  The regularizer $N(\cdot)$ prevents over-fitting.  For the first part of this paper we will restrict our attention to linear predictors and with some abuse of notation we will write $\mbf{f}(\mbf{x}) = \mbf{f}^T \mbf{x}$.

\subsection{Assumptions on loss and regularizer}

The conditions under which we can prove results on privacy and generalization error depend on analytic properties of the loss and regularizer.  In particular, we will require certain forms of convexity (see \cite{RockafellarW:98variational}).

\begin{definition}
A function $H(\mbf{f})$ over $\mbf{f} \in \bbR^d$ is said to be {\em strictly convex} if for all $\alpha \in (0,1)$, $\mbf{f}$, and $\mbf{g}$,
	\begin{align} 
	H\left( \alpha \mbf{f} + (1 - \alpha) \mbf{g} \right) < \alpha H(\mbf{f}) + (1 - \alpha) H(\mbf{g}).
	\end{align}
It is said to be {\em $\lambda$-strongly convex} if for all $\alpha \in (0,1)$, $\mbf{f}$, and $\mbf{g}$, 
	\begin{align}
	H\left( \alpha \mbf{f} + (1 - \alpha) \mbf{g} \right) 
		\leq \alpha H(\mbf{f}) + (1 - \alpha) H(\mbf{g})
			- \frac{1}{2} \lambda \alpha (1 - \alpha) \norm{\mbf{f} - \mbf{g}}_2^2.
	\end{align}
\end{definition}

A strictly convex function has a unique minimum -- see~\cite{BoydV:04convex}.  Strong convexity plays a role in guaranteeing our privacy and generalization requirements.  For our privacy results to hold we will also require that the regularizer $N(\cdot)$ and loss function $\loss(\cdot,\cdot)$ be differentiable functions of $\mbf{f}$.  This excludes certain classes of regularizers, such as the $\ell_1$-norm regularizer $N(\mbf{f}) = \norm{\mbf{f}}_1$, and classes of loss functions such as the hinge loss $\lsvm(\mbf{f}^T \mbf{x}, y) = (1 - y \mbf{f}^T \mbf{x})^{+}$.  In some cases we can prove privacy guarantees for approximations to these non-differentiable functions.

\subsection{Privacy model}

We are interested in producing a classifier in a
manner that preserves the privacy of individual entries of the dataset
$\mc{D}$ that is used in training the classifier.  The notion of privacy we use is the
\textit{$\epsilon$-differential privacy model}, developed by \cite{DworkMNS:06sensitivity,D06}, which defines a notion of privacy 
for a randomized algorithm $\mc{A}(\mc{D})$.  Suppose $\mc{A}(\mc{D})$
produces a classifier, and let $\mc{D}'$ be another dataset
that differs from $\mc{D}$ in one entry (which we assume is the private value of one person).  That is, $\mc{D}'$ and $\mc{D}$ have
$n-1$ points $(\mbf{x}_i, y_i)$ in common.   The algorithm $\mc{A}$ provides
differential privacy if for any set $\SS$, the likelihood that
$\mc{A}(\mc{D}) \in \SS$ is 
close to the likelihood $\mc{A}(\mc{D}') \in \SS$, (where the likelihood is over the randomness in the algorithm).  That is, any single entry  of the dataset
does not affect the output distribution of the algorithm by much; dually, this means that an adversary, who knows all but one entry of the dataset, cannot gain much additional information about the last entry by observing the output of the algorithm.

The following definition of differential privacy is due
to~\cite{DworkMNS:06sensitivity}, paraphrased from~\cite{WassermanZ:10framework}.
\begin{definition}
An algorithm $\mc{A}(\mc{B})$ taking values in a set $\mc{T}$ provides $\priveps$-differential privacy if 
\begin{align}
\sup_{\mc{S}} \sup_{\mc{D}, \mc{D}'} \frac{ \mu\left( \mc{S} ~|~ \mc{B} = \mc{D} \right) 
	}{
\mu\left( \mc{S} ~|~ \mc{B} = \mc{D}' \right) 
	}	
	\le 
	e^{\priveps},
	\label{eq:privdefdens}
\end{align}
where the first supremum is over all measurable $\mc{S} \subseteq \mc{T}$, the second is
over all datasets $\mc{D}$ and $\mc{D}'$ differing in a single entry, and
$\mu(\cdot|\mc{B})$ is the conditional distribution (measure) on $\mc{T}$
induced by the output $\mc{A}(\mc{B})$ given a dataset $\mc{B}$.  The ratio is interpreted to be 1 whenever the numerator and denominator are both 0.
\label{def:densitypriv}
\end{definition}

\begin{figure}
\centering
\includegraphics{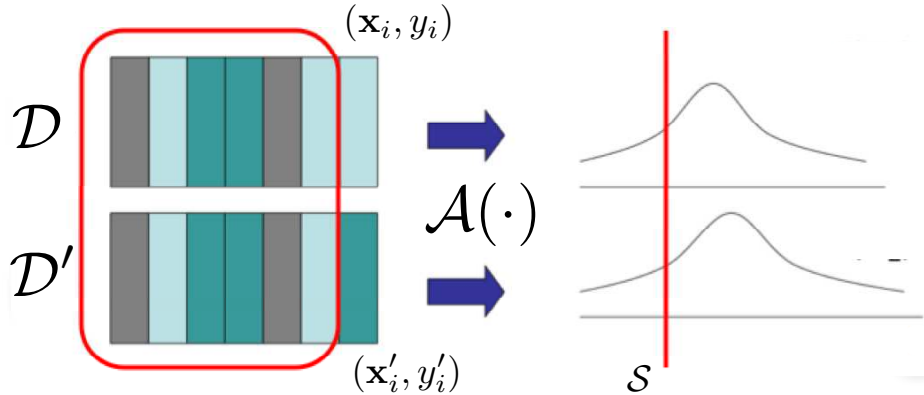}
\caption{An algorithm which is differentially private.  When datasets which are identical except for a single entry are input to the algorithm $\mc{A}$, the two distributions on the algorithm's output are close.  For a fixed measurable $\mc{S}$ the ratio of the measures (or densities) should be bounded.}
\label{fig:DPmodel}
\end{figure}

Note that if $\mc{S}$ is a set of measure 0 under the conditional measures induced by $\mc{D}$ and $\mc{D}'$, the ratio is automatically 1.  A more measure-theoretic definition is given in \cite{ZhouLW:09compression}.  An illustration of the definition is given in Figure \ref{fig:DPmodel}.

The following form of the definition is due to~\cite{DworkKMMN:06ourselves}.

\begin{definition}
An  algorithm $\mc{A}$ provides $\priveps$-differential privacy if for any two datasets $\mc{D}$ and $\mc{D}'$ that differ in a single entry and for any set $\SS$, 
	\begin{align}
	\exp(-\priveps) \prob(\mc{A}(\D') \in \SS) 
	\le \prob(\mc{A}(\D) \in \SS) \leq \exp(\priveps) \prob(\mc{A}(\D') \in \SS),
	\label{eq:privdef}
	\end{align}
where $\mc{A}(\D)$ (resp. $\mc{A}(\D')$) is the output of $\mc{A}$ on input $\D$ (resp. $\D'$).
\label{def:probpriv}
\end{definition}

We observe that an algorithm
$\mc{A}$ that satisfies Equation~\ref{eq:privdefdens} also satisfies
Equation~\ref{eq:privdef}, and as a result,
Definition~\ref{def:densitypriv} is stronger than
Definition~\ref{def:probpriv}.

From this definition, it is clear that the $\mc{A}(\mc{D})$ that outputs the
minimizer of the ERM objective (\ref{eq:StdRegERM}) does not provide $\priveps$-differential privacy for any $\priveps$.  This is because an ERM solution is a linear combination of some selected training samples ``near'' the decision boundary.  If
$\mc{D}$ and $\mc{D}'$ differ in one of these samples, then the classifier will change completely, making the likelihood ratio in (\ref{eq:privdef}) infinite.  Regularization helps by penalizing the $L_2$ norm of the change, but does not account how the direction of the minimizer is sensitive to changes in the data.

\cite{DworkMNS:06sensitivity} also provide a standard recipe for computing
privacy-preserving approximations to functions by adding noise with a
particular distribution to the output of the function. We call this recipe the
{\em sensitivity method.} Let $g: (\bbR^m)^n \rightarrow \bbR$ be a scalar 
function of $z_1,
\ldots, z_n$, where $z_i \in \bbR^m$ corresponds to the private value of
individual $i$; then the sensitivity of $g$ is defined as follows. 

\begin{definition}
The sensitivity of a function $g: (\bbR^m)^n \rightarrow \bbR$ is maximum
difference between the values of the function when one input changes. More
formally, the sensitivity $S(g)$ of $g$ is defined as:
	\begin{align}
	S(g) = \max_i \max_{z_1, \ldots, z_n, z'_i} 
		\left|
			g(z_1, \ldots, z_{i-1}, z_i, z_{i+1},\ldots, z_n) 
			- g(z_1, \ldots, z_{i-1}, z'_i, z_{i+1}, \ldots, z_n)
		\right|.
	\end{align}
\end{definition}

To compute a function $g$ on a dataset $\mc{D} = \{z_1, \ldots, z_n\}$, the sensitivity method
outputs $g(z_1, \ldots, z_n) + \eta$, where $\eta$ is a random variable
drawn according to the Laplace distribution, with mean $0$ and standard
deviation $\frac{S(g)}{\priveps}$. It is shown
in~\cite{DworkMNS:06sensitivity} that such a procedure is
$\priveps$-differentially private.

\section{Privacy-preserving ERM}
\label{sec:privacy}

Here we describe two approaches for creating privacy-preserving algorithms from (\ref{eq:StdRegERM}).

\subsection{Output perturbation : the sensitivity method}

Algorithm \ref{alg:output} is derived from the \textit{sensitivity method}
of \cite{DworkMNS:06sensitivity}, a general method for generating a
privacy-preserving approximation to any function $A(\cdot)$.  In this section the norm $\|\cdot\|$ is the $L_2$-norm unless otherwise specified.  For the
function $A(\mc{D}) = \argmin \obj(\mbf{f}, \mc{D})$, Algorithm 1 outputs a vector
$A(\mc{D}) + \b$, where $\b$ is random noise with density
	\begin{align}
	\nu(\b) = \frac{1}{\alpha} e^{- \beta\norm{\b}}~,
	\label{eq:density}
	\end{align}
where $\alpha$ is a normalizing constant.  The parameter $\beta$ is a function of $\priveps$, and the $L_2$-\textit{sensitivity} of $A(\cdot)$, which is defined as follows.

\begin{definition}
The $L_2$-sensitivity of a vector-valued function
is defined as the maximum change in the $L_2$ norm of the value of
the function when one input changes. More formally,
	\begin{align} 
	S(A) = \max_i \max_{z_1, \ldots, z_n, z'_i} \left\|
		A(z_1, \ldots, z_i,\ldots) - A(z_1, \ldots, z'_i, \ldots)
		\right\|.
	\end{align}
\end{definition}

The interested reader is referred to \cite{DworkMNS:06sensitivity} 
for further details.  Adding 
noise to the output of $A(\cdot)$ has the effect of masking the effect of
any particular data point.  However, in some applications the sensitivity
of the minimizer $\argmin \obj(\mbf{f}, \mc{D})$  may be quite high, which
would require the sensitivity method to add noise with high variance.

\begin{algorithm}
\caption{ERM with output perturbation (sensitivity)}
\label{alg:output}
\begin{algorithmic}
\STATE \textbf{Inputs:} Data $\mc{D} = \{z_i\}$, parameters $\priveps$, $\reg$.
\STATE \textbf{Output:} Approximate minimizer $\fpriv$.
\STATE Draw a vector $\b$ according to (\ref{eq:density}) with $\beta
= \frac{n \reg \priveps}{2}$.
\STATE Compute $\fpriv = \argmin \obj(\mbf{f},\mc{D}) + \b$.
\end{algorithmic}
\end{algorithm}

\subsection{Objective perturbation}

A different approach, first proposed by \cite{ChaudhuriM:08nips}, is to add noise to the objective function itself and then produce the minimizer of the perturbed objective.  That is, we can minimize 
	\begin{align}
	\obj_{\mathrm{priv}}(\mbf{f},\mc{D}) 
		&= \obj(\mbf{f}, \mc{D}) + \frac{1}{n} \b^T \mbf{f},
	\label{eq:privobj}
	\end{align}
where $\b$ has density given by (\ref{eq:density}), with $\beta = \priveps$.  Note that the privacy parameter here does not depend on the sensitivity of the of the classification algorithm.

\begin{algorithm}
\caption{ERM with objective perturbation}
\label{alg:objective}
\begin{algorithmic}
\STATE \textbf{Inputs:} Data $\mc{D} = \{z_i\}$, parameters $\priveps$, $\reg$, $\c$.
\STATE \textbf{Output:} Approximate minimizer $\fpriv$.
\STATE Let $\priveps' = \priveps - \log(1 + \frac{2\c}{n \reg} +
\frac{\c^2}{n^2 \reg^2})$. 
\STATE If $\priveps' > 0$, then $\extra = 0$, else
$\extra = \frac{c}{n(e^{\priveps/4} - 1)} - \reg$, and $\priveps' =
\priveps/2$. 
\STATE Draw a vector $\b$ according to (\ref{eq:density}) with $\beta
= \priveps'/2$.
\STATE Compute $\fpriv = \argmin \obj_{\mathrm{priv}}(\mbf{f},\mc{D}) +
\frac{1}{2} \extra ||\f||^2$.
\end{algorithmic}
\end{algorithm}

The algorithm requires a certain slack, $\log(1 + \frac{2\c}{n \reg} +
\frac{\c^2}{n^2 \reg^2})$, in the privacy parameter.  This is due to additional factors in bounding the ratio of the densities.  The ``If'' statement in the algorithm is from having to consider two cases in the proof of Theorem \ref{thm:objectivepriv}, which shows that the algorithm is differentially private.

\subsection{Privacy guarantees}

In this section, we establish the conditions under which Algorithms 1 and 2 provide
$\priveps$-differential privacy. First, we establish guarantees for
Algorithm 1.

\subsubsection{Privacy Guarantees for Output Perturbation}

\begin{theorem}
If $N(\cdot)$ is differentiable, and $1$-strongly convex, and
$\loss$ is convex and differentiable, with $|\loss'(z)| \leq 1$ for all $z$, then,
Algorithm 1 provides $\priveps$-differential privacy.
\label{thm:sensitivitypriv}
\end{theorem}

The proof of Theorem~\ref{thm:sensitivitypriv} follows from
Corollary~\ref{cor:sensitivity}, and~\cite{DworkMNS:06sensitivity}.
The proof is provided here for completeness.

\begin{proof}
From Corollary \ref{cor:sensitivity}, if the conditions on $N(\cdot)$ and
$\loss$ hold, then the $L_2$-sensivity of ERM with regularization parameter $\reg$ is at most $\frac{2}{n \reg}$.
We observe that when we pick $||\b||$ from the
distribution in Algorithm 1, for a specific vector $\mbf{b_0} \in \bbR^d$, the density
at $\mbf{b_0}$ is proportional to $e^{-\frac{n \reg \priveps}{2}
||\mbf{b_0}||}$. Let $\D$ and $\D'$ be any two datasets that differ in the
value of one individual. Then, for any $\mbf{f}$, 
	\begin{align} 
	\frac{g(\mbf{f} | \D)}{g(\mbf{f} | \D')} 
	= \frac{\nu(\b_1)}{\nu(\b_2)} 
	= e^{-\frac{n \reg \priveps}{2}(||\b_1|| - ||\b_2||)},
	\end{align}
where $\b_1$ and $\b_2$ are the corresponding noise vectors chosen in Step 1 of
Algorithm 1, and $g(\f | \D)$ ($g(\f | \D')$ respectively) is the density of the
output of Algorithm 1 at $\f$, when the input is $\D$ ($\D'$ respectively).
If $\f_1$ and $\f_2$ are the solutions respectively to non-private
regularized ERM when the input is $\D$ and $\D'$,
then, $\b_2 - \b_1
= \f_2 - \f_1$. From Corollary \ref{cor:sensitivity}, and using a
triangle inequality,
	\begin{align} 
	||\b_1|| - ||\b_2|| 
	\leq ||\b_1 - \b_2 || 
	= ||\f_2 -\f_1|| 
	\leq \frac{2}{n \reg}.
	\end{align}
Moreover, by symmetry, the density of the directions of $\b_1$ and
$\b_2$ are uniform. Therefore, by construction,
$\frac{\nu(\b_1)}{\nu(\b_2)} \leq e^{\priveps}$. The theorem follows.
\end{proof}

The main ingredient of the proof of Theorem~\ref{thm:sensitivitypriv} is a
result about the sensitivity of regularized ERM, which is provided below.

\begin{lemma}   \label{lem:sensitivity}
Let $G(\mbf{f})$ and $g(\mbf{f})$ be two vector-valued functions, 
which are continuous, and differentiable at all points. Moreover, let $G(\f)$ and $G(\f) + g(\f)$ be
$\lambda$-strongly convex. If $\f_1 = \argmin_\mbf{f} G(\mbf{f})$ and $\f_2 = \argmin_{\mbf{f}} G(\mbf{f}) + g(\mbf{f})$, then
	\begin{align}
	\norm{\f_1 - \f_2} \leq \frac{1}{\lambda} \max_{\mbf{f}} \norm{ \nabla g(\mbf{f}) }.
	\end{align}
\end{lemma}

\begin{proof}
Using the definition of $\f_1$ and $\f_2$, and the fact that $G$ and $g$ are
continuous and differentiable everywhere, 
	\begin{equation}
	\nabla G(\f_1) = \nabla G(\f_2) + \nabla g(\f_2) = {\mbf{0}}.
	\label{eqn:derv}
	\end{equation}

As $G(\f)$ is $\lambda$-strongly convex, it follows from Lemma 14 of \cite{Shaithesis} that:
	\begin{align}
	(\nabla G(\f_1) - \nabla G(\f_2))^T (\f_1 - \f_2) 
		\geq \lambda \norm{\f_1 - \f_2}^2.
	\end{align}
Combining this with (\ref{eqn:derv}) and the Cauchy-Schwartz inequality, we get that
	\begin{align}
	\norm{\f_1 - \f_2} \cdot \norm{\nabla g(\f_2)} 
		\geq (\f_1 - \f_2)^T \nabla g(\f_2) 
		=  (\nabla G(\f_1) - \nabla G(\f_2))^T (\f_1 - \f_2) 
		\geq \lambda \norm{\f_1 - \f_2}^2.
	\end{align}
The conclusion follows from dividing both sides by $\lambda \norm{\f_1 - \f_2}$.
\end{proof}

\begin{corollary}
If $N(\cdot)$ is differentiable and $1$-strongly convex, and $\loss$ is
convex and differentiable with $|\loss'(z)| \leq 1$ for all $z$, then, the $L_2$-sensitivity of
$\obj(\mbf{f}, \mc{D})$ is at most $\frac{2}{n \reg}$.
\label{cor:sensitivity}
\end{corollary}

\begin{proof}
Let $\mc{D} = \{ (\x_1, y_1), \ldots, (\x_n, y_n) \}$ and $\mc{D'} = \{ (\x_1,
y_1), \ldots, (\x'_n, y'_n) \}$ be two datasets that differ in the value of
the $n$-th individual. Moreover, we let $G(\f) = \obj(\mbf{f}, \mc{D})$, 
$g(\f) = \obj(\f, \D') - \obj(\f, \D)$, $\f_1 =  \argmin_{\mbf{f}} \obj(\mbf{f},
\mc{D})$, and $\f_2 = \argmin_{\f} \obj(\f, \D')$. Finally, we set
$g(\f) = \frac{1}{n}(\loss(y'_n \f^T \x'_n) - \loss(y_n \f^T \x_n))$.

We observe that due to the convexity of $\loss$, and $1$-strong convexity
of $N(\cdot)$, $G(\f) = \obj(\f,\D)$ is $\reg$-strongly convex. Moreover, 
$G(\f) + g(\f) = \obj(\f, \D')$ is also $\reg$-strongly convex. Finally,
due to the differentiability of $N(\cdot)$ and $\loss$, $G(\f)$ and $g(\f)$
are also differentiable at all points.  We have:
	\begin{align} 
	\nabla g(\f) = \frac{1}{n} ( y_n \loss'(y_n \mbf{f}^T \x_n) \x_n 
		- y'_n \loss'(y'_n \mbf{f}^T \x'_n) \x'_n).
	\end{align}
As $y_i \in [-1, 1]$, $|\loss'(z)| \leq 1$, for all $z$, and $||\x_i||\leq
1$, for any $\f$, $||\nabla g(\f)|| \leq \frac{1}{n}(||\x_n - \x'_n||) \leq \frac{1}{n}(||\x_n|| +
||\x'_n||) \leq \frac{2}{n}$.
The proof now follows by an application of Lemma~\ref{lem:sensitivity}.
\end{proof}

\subsubsection{Privacy Guarantees for Objective Perturbation}

In this section, we show that Algorithm 2 is $\priveps$-differentially
private.  This proof requires stronger assumptions on the loss function 
than were required in Theorem \ref{thm:sensitivitypriv}.  In certain cases,
some of these assumptions can be weakened; for such an example, see
Section~\ref{sec:svmprivacy}.

\begin{theorem}
If $N(\cdot)$ is $1$-strongly convex and doubly differentiable,
and $\loss(\cdot)$ is convex and doubly differentiable, 
with $|\loss'(z)| \leq 1$ and $|\loss''(z)| \leq \c$ for
all $z$, then Algorithm 2 is $\priveps$-differentially private.
\label{thm:objectivepriv}
\end{theorem}

\begin{proof}
Consider an $\fpriv$ output by Algorithm 2. We observe that given {\em any} fixed
$\fpriv$ and a fixed dataset $\D$, there always exists a $\b$ such that
Algorithm 2 outputs $\fpriv$ on input $\D$.  Because $\loss$ is differentiable 
and convex, and $N(\cdot)$ is differentiable, 
we can take the gradient of the objective function and set it to
$\mbf{0}$ at $\fpriv$. Therefore,
	\begin{equation}
	\b = -n \reg \nabla N(\fpriv) 
		- \sum_{i=1}^{n} y_i \loss'(y_i \fpriv^T \mbf{x}_i) \mbf{x}_i 
		- n \extra \fpriv 
		\label{eqn:bf}.
	\end{equation}
Note that (\ref{eqn:bf}) holds because for any $\mbf{f}$, $\nabla
\loss(\mbf{f}^T \mbf{x}) = \loss'(\mbf{f}^T \mbf{x}) \mbf{x}$.  

We claim that as $\loss$ is differentiable and $\obj(\f, \D) + \frac{\extra}{2}
||\f||^2$ is strongly convex, given a dataset $\mc{D} = (\x_1, y_1),
\ldots, (\x_n, y_n)$, there is a bijection between $\b$ and $\fpriv$.  The relation (\ref{eqn:bf}) shows that two different $\b$ values cannot result in the same $\fpriv$.  Furthermore, since the objective is strictly convex, for a
fixed $\b$ and $\mc{D}$, there is a unique $\fpriv$; therefore the map
from $\b$ to $\fpriv$ is injective.  The relation (\ref{eqn:bf}) also shows that for any $\fpriv$ there exists a $\b$ for which $\fpriv$ is the minimizer, so the map from $\b$ to $\fpriv$ is surjective.

To show $\priveps$-differential privacy, we need to compute the ratio $g( \fpriv | \mc{D} )/g( \fpriv |
\mc{D}' )$ of
the densities of $\fpriv$ under the two datasets $\mc{D}$ and $\mc{D'}$. This ratio can be
written as \citep{Billingsley}
\begin{align*}
	\frac{ g( \fpriv | \mc{D} )}{ g( \fpriv | \mc{D}' ) } 
	&= \frac{ \mu( \b | \mc{D} )}{ \mu( \b' | \mc{D}' ) } 
		\cdot \frac{ |\det(\mbf{J}(\fpriv \to \b |
		\mc{D}))|^{-1} }{ |\det(\mbf{J}(\fpriv \to \b' |
		\mc{D}'))|^{-1} },
\end{align*}
where $\mbf{J}(\fpriv \to \b | \mc{D})$, $\mbf{J}(\fpriv \to \b
| \mc{D'})$ are the Jacobian matrices of the mappings from $\fpriv$ to $\b$,
and $\mu(\b | \mc{D})$ and $\mu(\b | \mc{D'})$ are the densities
of $\b$ given the output $\fpriv$, when the datasets are $\mc{D}$ and $\mc{D'}$ respectively.

First, we bound the ratio of the Jacobian determinants.   Let $\b^{(j)}$ denote the $j$-th coordinate of $\b$.  From (\ref{eqn:bf}) we have
	\begin{equation*}
	\b^{(j)} = - n \reg \nabla N(\fpriv)^{(j)} 
		- \sum_{i=1}^{n} \loss'(y_i \fpriv^T 	\mbf{x}_i) \mbf{x}_i^{(j)} 
		- n\extra \fpriv^{(j)}.
	\end{equation*}
Given a dataset $\D$, the $(j,k)$-th entry of the Jacobian
matrix $\mbf{J}(\mbf{f} \to \b | \mc{D})$ is
	\begin{align*}
	\frac{\partial \b^{(j)} }{\partial \fpriv^{(k)}} 
		= -n \reg \nabla^2 N(\fpriv)^{(j, k)} 
			- \sum_{i} y_i^2 \loss''(y_i \fpriv^T \mbf{x}_i) \mbf{x}_i^{(j)} \x_i^{(k)} 
			- n \extra 1(j=k),
	\end{align*}
where $1(\cdot)$ is the indicator function. We note that the Jacobian is
defined for all $\fpriv$ because $N(\cdot)$ and $\loss$ are globally
doubly differentiable.
 
Let $\mc{D}$ and $\mc{D'}$ be two datasets which differ in the value of the $n$-th
item such that \\
$\mc{D} = \{ (\x_1, y_1), \ldots, (\x_{n-1}, y_{n-1}), (\x_n, y_n)
\}$ and $\mc{D'} = \{ (\x_1, y_1), \ldots, (\x_{n-1}, y_{n-1}), (\x'_n,
y'_n) \}$.
Moreover, we define matrices $A$ and $E$ as follows:
	\begin{align*}
	A & =  n \reg \nabla^2 N(\fpriv) 
		+ \sum_{i=1}^{n} y_i^2 \loss''(y_i \fpriv^T \mbf{x}_i) \mbf{x}_i \mbf{x}_i^T  
		+ n \extra I_d\\
	E & = - y_n^2 \loss''(y_n \fpriv^T \mbf{x}_n) \mbf{x}_n \mbf{x}_n^T 
		+ (y'_n)^2 \loss''(y'_n \fpriv^T \mbf{x'}_n) \mbf{x'}_n \mbf{x'}_n^T.
	\end{align*}
Then, $\mbf{J}(\fpriv \to \b | \mc{D}) = -A$, and $\mbf{J}(\fpriv \to
\b | \mc{D'}) = -(A + E)$. 

Let $\lambda_1(M)$ and $\lambda_2(M)$ denote the largest and second largest eigenvalues of a matrix $M$.  As $E$ has rank at most $2$, from Lemma~\ref{lem:detlambda},
	\begin{align*}
	\frac{ |\det( \mbf{J}(\fpriv \to \b | \mc{D'}))| } { |\det(
	\mbf{J}(\fpriv \to \b | \mc{D}))|} 
	&=
	\frac{|\det(A+E)|}{|\det{A}|} \\
	&= |1 + \lambda_1(A^{-1}E) +
	\lambda_2(A^{-1}E) + \lambda_1(A^{-1}E) \lambda_2(A^{-1}E)|.
	\end{align*}
For a $1$-strongly convex function $N$, the Hessian $\nabla^2 N(\fpriv)$
has eigenvalues greater than $1$ \citep{BoydV:04convex}.  Since we have
assumed $\loss$ is doubly differentiable and convex, any eigenvalue of $A$ is
therefore at least $n\reg + n \extra$; therefore, for $j=1,2$,
$|\lambda_j(A^{-1}E)| \leq \frac{|\lambda_j(E)|}{n (\reg + \extra)}$. 
Applying the triangle inequality to the trace norm:
	\begin{align*}
	|\lambda_1(E)| + |\lambda_2(E)| 
		\le |y_n^2 \loss''(y_n \fpriv^T \x_n)| \cdot \norm{\x_n}
			+ | -(y'_n)^2 \loss''(y'_n \fpriv^T \x'_n) | \cdot \norm{\x'_n}.
	\end{align*}
Then upper bounds on $|y_i|$, $||\x_i||$, and $|\loss''(z)|$ yield
	\begin{align*}
	|\lambda_1(E)| + |\lambda_2(E)| \leq 2\c.
	\end{align*}

Therefore, $|\lambda_1(E)|\cdot |\lambda_2(E)| \leq c^2$, and 
	\begin{align*}
	\frac{|\det(A+E)|}{|\det(A)|} \leq 1 + \frac{2\c}{n (\reg + \extra)} +
\frac{c^2}{n^2 (\reg + \extra)^2}  = \left(1 + \frac{\c}{n (\reg +
\extra)}\right)^2.
	\end{align*}
We now consider two cases. In the first case, $\extra = 0$, and by
definition, in that case, $1 + \frac{2 \c}{n \reg} + \frac{\c^2}{n^2
\reg^2} \leq e^{\priveps - \priveps'}$. In the second case, $\extra > 0$,
and in this case, by definition of $\extra$, $(1 + \frac{\c}{n (\reg +
\extra)})^2 = e^{\priveps/2} = e^{\priveps - \priveps'}$.
 
Next, we bound the ratio of the densities of $\b$. We observe that as
$|\loss'(z)| \leq 1$, for any $z$ and $|y_i|, ||\x_i|| \leq 1$, for datasets
$\mc{D}$ and $\mc{D'}$ which differ by one value,
	\begin{align*}
	\mbf{b'} - \b = y_n \loss'(y_n \fpriv^T \mbf{x}_n) \mbf{x}_n 
		-  y'_n \loss'(y_n \fpriv^T \mbf{x'}_n) \mbf{x'}_n.
	\end{align*}
This implies that:
	\begin{align*} 
		\norm{\b} - \norm{\mbf{b'}} 
		\leq 
		\norm{\b - \mbf{b'}} 
		\leq 2.
	\end{align*}
We can write:
	\begin{align*} 
	\frac{ \mu(\b | \mc{D})}{\mu(\mbf{b'} | \mc{D'})} 
	=
	\frac{||\b||^{d-1} e^{-\priveps' ||\b||/2} 
			\cdot \frac{1}{\surf(||\b||)}
		}{
		||\mbf{b'}||^{d-1} e^{-\priveps' ||\mbf{b'}||/2} 
			\cdot \frac{1}{\surf(||\mbf{b'}||)}} 
	\leq e^{\priveps'(||\b|| - ||\mbf{b'}||)/2} \leq e^{\priveps'},
	\end{align*}
where $\surf(x)$ denotes the surface area of the sphere in $d$ dimensions with
radius $x$. Here the last step follows from the fact that $\surf(x) =
s(1)x^{d-1}$, where $s(1)$ is the surface area of the unit sphere in $\bbR^d$.

Finally, we are ready to bound the ratio of densities:
	\begin{align*}
	\frac{ g( \fpriv | \mc{D} )}{ g( \fpriv | \mc{D}' ) } 
	&= \frac{ \mu( \b | \mc{D} )}{ \mu( \b' | \mc{D}' ) } 
		\cdot \frac{ |\det(\mbf{J}(\fpriv \to \b | \mc{D}'))| }{
		|\det(\mbf{J}(\fpriv \to \b' | \mc{D}))| } \\
	&= \frac{ \mu( \b | \mc{D} )}{ \mu( \b' | \mc{D}' ) } 
		\cdot \frac{|\det(A + E)|}{|\det{A}|} \\
	&\leq e^{\priveps'} \cdot e^{\priveps - \priveps'} \\
	&\leq e^{\priveps}.
	\end{align*}
Thus, Algorithm 2 satisfies Definition~\ref{def:densitypriv}.
\end{proof}

\begin{lemma}
If $A$ is full rank, and if $E$ has rank at most $2$, then,
\begin{align} 
	\frac{\det(A+E) - \det(A)}{\det(A)} 
		= \lambda_1(A^{-1}E) + \lambda_2(A^{-1}E) 
			+ \lambda_1(A^{-1}E)\lambda_2(A^{-1}E),
\end{align}
where $\lambda_j(Z)$ is the $j$-th eigenvalue of matrix $Z$.
\label{lem:detlambda}
\end{lemma}

\begin{proof}
Note that $E$ has rank at most $2$, so $A^{-1} E$ also has rank at most $2$.  Using the fact that $\lambda_i(I + A^{-1}E) = 1 + \lambda_i(A^{-1} E)$,
	\begin{align*}
	\frac{\det(A + E) - \det(A)}{\det{A}} 
	&= \det(I + A^{-1} E) - 1 \\
	&= (1 + \lambda_1(A^{-1} E)) (1 + \lambda_2(A^{-1} E) ) - 1 \\
	&= \lambda_1(A^{-1}E) + \lambda_2(A^{-1}E) 
			+ \lambda_1(A^{-1}E)\lambda_2(A^{-1}E).
	\end{align*}
\end{proof}

\subsection{Application to classification}

In this section, we show how to use our results to provide privacy-preserving versions of logistic regression and support vector machines.

\subsubsection{Logistic Regression \label{sec:PrivLR}}

One popular ERM classification algorithm is
regularized logistic regression. In this case, $N(\f) = \frac{1}{2}
||\f||^2$, and the loss function is $\llr(z) = \log(1 + e^{-z})$.
Taking derivatives and double derivatives,
	\begin{align*}
	\llr'(z) & = \frac{-1}{(1 + e^{z})} \nonumber \\
	\llr''(z) & = \frac{1}{(1 + e^{-z})(1 + e^z)}.
	\label{eqn:lr}
	\end{align*}
Note that $\llr$ is continuous, differentiable and doubly differentiable, with $\c \leq \frac{1}{4}$. Therefore, we can plug in logistic loss directly to Theorems~\ref{thm:sensitivitypriv} and ~\ref{thm:objectivepriv} to get the following result.

\begin{corollary}
The output of Algorithm 1 with $N(\f)= \frac{1}{2} ||\f||^2$, $\loss = \llr$ is an $\priveps$-differentially private approximation to logistic regression.  The output of Algorithm 2 with $N(\f) = \frac{1}{2} ||\f||^2$, $\c=\frac{1}{4}$, and $\loss = \llr$, is an $\priveps$-differentially private approximation to logistic regression.
\label{cor:lr}
\end{corollary}

We quantify how well the outputs of Algorithms 1 and 2 approximate
(non-private) logistic regression in Section~\ref{sec:utility}.

\subsubsection{Support Vector Machines}
\label{sec:svmprivacy}
Another very commonly used classifier is $L_2$-regularized support vector
machines. In this case, again, $N(\f) = \frac{1}{2} ||\f||^2$, and
	\begin{align} 
	\lsvm(z) = \max(0, 1 - z). 
	\end{align}
Notice that this loss function is continuous, but not differentiable, and thus it does not satisfy conditions in Theorems~\ref{thm:sensitivitypriv} and~\ref{thm:objectivepriv}. 

There are two alternative solutions to this. First, we
can approximate $\lsvm$ by a different loss function, which is doubly
differentiable, as follows (see also \cite{Chapelle:07primal}):
\begin{align}
	\loss_{s}(z) = \left\{
		\begin{array}{lll}
		0 & \textrm{if} & z > 1+h \\
		-\frac{(1 - z)^4}{16h^3} + \frac{3(1 - z)^2}{8h} + \frac{1 - z}{2} + \frac{3h}{16} & \textrm{if} & |1 - z| \le h \\
		1 - z & \textrm{if} & z < 1 - h
		\end{array}
		\right.
	\end{align}
As $h \rightarrow 0$, this loss approaches the hinge loss.
Taking derivatives, we observe that:
\begin{align}
	\loss'_{s}(z) = \left\{
		\begin{array}{lll}
		0 & \textrm{if} & z > 1+h \\
		\frac{(1 - z)^3}{4h^3} - \frac{3(1 - z)}{4h} - \frac{1}{2}  & \textrm{if} & |1 - z| \le h \\
		-1 & \textrm{if} & z < 1 - h
		\end{array}
		\right.
	\end{align}
Moreover,
\begin{align}
	\loss''_{s}(z) = \left\{
		\begin{array}{lll}
		0 & \textrm{if} & z > 1+h \\
		-\frac{3(1 - z)^2}{4h^3} + \frac{3}{4h}  & \textrm{if} & |1 - z| \le h \\
		0 & \textrm{if} & z < 1 - h
		\end{array}
		\right.
	\end{align}
Observe that this implies that $|\loss_{s}''(z)| \leq \frac{3}{4h}$ for all
$h$ and $z$. Moreover, $\loss_s$ is convex, as $\loss_{s}''(z) \geq 0$ for all $z$.
Therefore, $\loss_{s}$ can be used in Theorems~\ref{thm:sensitivitypriv} and ~\ref{thm:objectivepriv}, which
gives us privacy-preserving approximations to regularized support vector machines.

\begin{corollary}
The output of Algorithm 1 with $N(\f) = \frac{1}{2}||\f||^2$, and $\loss = \loss_{s}$ is an $\priveps$-differentially private approximation to support vector machines.
The output of Algorithm 2 with $N(\f) = \frac{1}{2}||\f||^2$, $\c=\frac{3}{4h}$, and $\loss = \loss_{s}$ is an $\priveps$-differentially private approximation to support vector machines.
\label{cor:svm}
\end{corollary}

The second solution is to use Huber Loss, as suggested by~\cite{Chapelle:07primal}, which is defined as follows:
 \begin{align}
	\lhuber(z) = \left\{
		\begin{array}{lll}
		0 & \textrm{if} & z > 1+h \\
		\frac{1}{4h}(1 + h - z)^2 & \textrm{if} & |1 - z| \le h \\
		1 - z & \textrm{if} & z < 1 - h
		\end{array}
		\right.
	\label{eq:Huber}
	\end{align}
Observe that Huber loss is convex and differentiable, and piecewise
doubly-differentiable, with $c=\frac{1}{2h}$. However, it is not globally doubly
differentiable, and hence the Jacobian in the proof of
Theorem~\ref{thm:objectivepriv} is undefined for certain values of $\f$.
However, we can show that in this case, Algorithm 2, when run with $\c = \frac{1}{2h}$ satisfies Definition~\ref{def:probpriv}. 

Let $G$ denote the map from $\fpriv$ to $\b$ in (\ref{eqn:bf})  
under $\mc{B} = \D$, and $H$ denote the map under $\mc{B} = \D'$.  By  
definition, the probability $\prob( \fpriv \in \SS ~|~ \mc{B} = \D ) =  
\prob_{\b}( \b \in G(\SS) )$.

\begin{corollary}
Let $\fpriv$ be the output of Algorithm 2 with $\loss = \lhuber 
$, $c=\frac{1}{2h}$, and $N(\f) = \frac{1}{2} || \f ||_2^2$.  For any  
set $\SS$ of possible values of $\fpriv$, and any pair of datasets $\D 
$, $\D'$ which differ in the private value of one person $(\x_n, y_n)$,
	\begin{align}
	e^{-\priveps} \prob( \SS ~|~ \mc{B} = \D') \le \prob(\SS ~|~ \mc{B} =  
\D) \leq e^{\priveps} \prob(\SS ~|~ \mc{B} = \D').
	\end{align}
\label{cor:huber}
\end{corollary}

\begin{proof}
Consider the event $\fpriv \in \SS$.  Let $\mc{T} = G(\SS)$ and $ 
\mc{T}' = H(\SS)$.  Because $G$ is a bijection, we have
	\begin{align}
	\prob( \fpriv \in \SS ~|~ \mc{B} = \D) = \prob_{\b}( \b \in  
\mc{T} ~|~ \mc{B} = \D),
	\end{align}
and a similar expression when $\mc{B} = \D'$.

Now note that $\lhuber'(z)$ is only non-differentiable for a  
finite number of values of $z$. Let $\mc{Z}$ be the set of these  
values of $z$.
	\begin{align}
	\mc{C} = \{ \mbf{f} : y \mbf{f}^T \mbf{x} = z \in \mc{Z},\  
(\mbf{x},y) \in \D \cup \D' \}.
	\end{align}
Pick a tuple $(z, (\mbf{x}, y)) \in \mc{Z} \times (\D \cup \D')$.   
The set of $\f$ such that $y \mbf{f}^T  
\mbf{x} = z$ is a hyperplane in $\bbR^d$.  Since $\nabla N(\f) = \f/2$ and $\loss'$ is piecewise linear, from $(\ref{eqn:bf})$ we see that the set of corresponding $\mbf{b}$'s is also piecewise linear, and hence has Lebesgue measure $0$.  
Since the measure corresponding to $\b$ is  
absolutely continuous with respect to the Lebesgue measure, this  
hyperplane has probability 0 under $\b$ as well.  Since $\mc{C}$  
is a finite union of such hyperplanes, we have $\prob( \b \in  
G(\mc{C})) = 0$.

Thus we have $\prob_{\b}( \mc{T} ~|~ \mc{B} = \D) =  
\prob_{\b}( G(\SS \setminus \mc{C}) ~|~ \mc{B} = \D)$, and  
similarly for $\D'$.  From the definition of $G$ and $H$, for $\f \in  
\SS \setminus \mc{C}$,
	\begin{align}
	H(\f) = G(\f) + y_n \loss'( y_n \f^T \mbf{x}_n) \mbf{x}_n - y'_n  
\loss'( y'_n \f^T \mbf{x}'_n ) \mbf{x}'_n.
	\end{align}
since $\f \notin \mc{C}$, this mapping shows that if $\prob_{\b} 
( G(\SS \setminus \mc{C}) ~|~ \mc{B} = \D) = 0$ then we must have
$\prob_{\b}( H(\SS \setminus \mc{C}) ~|~ \mc{B} = \D) = 0$.  Thus  
the result holds for sets of measure $0$.  If $\SS \setminus \mc{C}$  
has positive measure we can calculate the ratio of the probabilities  
for $\fpriv$ for which the loss is twice-differentiable.  For such $ 
\fpriv$ the Jacobian is also defined,  and we can use a method similar  
to Theorem~\ref{thm:objectivepriv} to prove the result.
\end{proof}

\textbf{Remark:}  Because the privacy proof for Algorithm \ref{alg:output} does not require the analytic properties of \ref{alg:objective}, we can also use Huber loss in Algorithm \ref{alg:output} to get an $\geneps$-differentially private approximation to the SVM.  We quantify how well the outputs of Algorithms 1 and 2 approximate private support vector machines in Section~\ref{sec:utility}.  These approximations to the hinge loss are necessary because of the analytic requirements of Theorems~\ref{thm:sensitivitypriv} and~\ref{thm:objectivepriv} on the loss function.  Because the requirements of Theorem \ref{thm:objectivepriv} are stricter, it may be possible to use an approximate loss in Algorithm 1 that would not be admissible in Algorithm 2.

\section{Generalization performance}
\label{sec:utility}
In this section, we provide guarantees on the performance of privacy-preserving
ERM algorithms in Section~\ref{sec:privacy}.  We provide these bounds for
$L_2$-regularization.  To quantify this performance, we will assume that
the $n$ entries in the dataset $\D$ are drawn i.i.d. according to a fixed
distribution $P(\mbf{x},y)$.  We measure the performance of these
algorithms by the number of samples $n$ required to acheive error
$\bestloss + \geneps$, where $\bestloss$ is the loss of a reference ERM
predictor $\fopt$.  This resulting bound on $\geneps$ will depend on the norm $\norm{\fopt}$ of this predictor.  By choosing an upper bound $\nu$ on the norm, we can interpret the result as saying that the privacy-preserving classifier will have error $\geneps$ more than that of any predictor with $\norm{\fopt} \le \nu$.

Given a distribution $P$ the expected loss $L(\mbf{f})$ for a classifier $\mbf{f}$ is
	\begin{align}
	L(\f) = \expe_{(\x, y) \sim P}\left[ \loss(\f^T \x, y) \right].
	\end{align}
The sample complexity for generalization error $\geneps$ against a classifier $\fopt$ is number of samples $n$ required to achieve error $L(\fopt) + \geneps$ under any data distribution $P$.  We would like the sample complexity to be low.

For a fixed $P$ we define the following function, which will be useful in our analysis:
	\begin{align}
	\expobj(\f) = L(\f) + \frac{\reg}{2} \norm{\f}^2.
	\end{align}
The function $\expobj(\f)$ is the expectation (over $P$) of the non-private $L_2$-regularized ERM objective evaluated at $\f$.

For non-private ERM, \cite{ShalevShwartzS:08icml} show that for a given $\fopt$ with loss $L(\fopt) = \bestloss$, if the number of data points
	\begin{align}
	n > C \frac{||\fopt||^2\log(\frac{1}{\delta})}{\geneps^2}
	\end{align}
for some constant $C$, then the excess loss of the $L_2$-regularized SVM solution $\mbf{f}_{svm}$ satisfies $L(\mbf{f}_{svm}) \le L(\fopt) + \geneps$.  This order growth will hold for our results as well.  It also serves as a reference against which we can compare the additional burden on the sample complexity imposed by the privacy constraints.

For most learning problems, we require the generalization error $\geneps < 1$. Moreover, it is also typically the case that for more difficult learning problems,
$||\fopt||$ is higher. For example, for regularized SVM, $\frac{1}{||\fopt||}$ is the
margin of classification, and as a result, $||\fopt||$ is higher for
learning problems with smaller margin. From the bounds provided in this
section, we note that the dominating term in the sample requirement for
objective perturbation has a better dependence on $||\fopt||$ as well as
$\frac{1}{\geneps}$; as a result, for more difficult learning problems, we
expect objective perturbation to perform better than output perturbation.

\subsection{Output perturbation}

First, we provide performance guarantees for Algorithm \ref{alg:output}, by
providing a bound on the number of samples required for Algorithm \ref{alg:output} to
produce a classifier with low error. 

\begin{definition}
A function $g(z) : \bbR \to \bbR$ is $\c$-Lipschitz if for all pairs $(z_1,z_2)$ we have $|g(z_1) - g(z_2)| \leq \c |z_1 - z_2|$.
\end{definition}

Recall that if a function $g(z)$ is differentiable, with $|g'(z)| \leq r$ for all $z$, then $g(z)$ is also $r$-Lipschitz.

\begin{theorem}	\label{thm:sensaccuracy}
Let $N(\f) = \frac{1}{2}||\f||^2$, and let $\fopt$ be a classifier such that $L(\fopt) = \bestloss$, and let $\delta > 0$.  If $\loss$ is differentiable and continuous with $|\loss'(z)| \le 1$, the derivative $\loss'$ is $\c$-Lipschitz, the data $\D$ is drawn i.i.d. according to $P$, then there exists a constant $C$ such that if the number of training samples satisfies
	\begin{align} 
	n > C \max\left(
		\frac{||\fopt||^2\log(\frac{1}{\delta})}{\geneps^2},
		\frac{d \log (\frac{d}{\delta})||\fopt||}{\geneps \priveps}, 
		\frac{d\log(\frac{d}{\delta}) \c^{1/2} ||\fopt||^2}{\geneps^{3/2}\priveps}
		\right),
	\label{eq:outputGen}
	\end{align}
where $d$ is the dimension of the data space, then the output $\fpriv$ of Algorithm \ref{alg:output} satisfies
	\begin{align}
	\prob\left( L(\fpriv) \le \bestloss + \geneps \right) \ge 1 - 2 \delta.
	\end{align}
\end{theorem}

\begin{proof}
Let 
	\begin{align*}
	\frtr &= \argmin_{\f} \expobj(\f) \\
	\f^{\ast} &= \argmin_{\f} \obj(\f,\D),
	\end{align*} 
and $\fpriv$ denote the output of Algorithm \ref{alg:output}.  Using the analysis method of  \cite{ShalevShwartzS:08icml} shows
	\begin{align}
	L(\fpriv) &=  L(\fopt) + ( \expobj(\fpriv) - \expobj(\frtr)) 
	   + ( \expobj(\frtr) - \expobj(\fopt)) + \frac{\reg}{2} ||\fopt||^2 - \frac{\reg}{2} ||\fpriv||^2.
	\label{eqn:gen}
	\end{align}
We will bound the terms on the right-hand side of (\ref{eqn:gen}).

For a regularizer $N(\f) = \frac{1}{2} ||\f||^2$ the Hessian satisfies $||\nabla^2 N(\f)||_2 \leq 1$ . Therefore, from Lemma \ref{lem:trainerrorsens}, with probability $1 - \delta$ over the privacy mechanism,
	\begin{align*}
	\obj(\fpriv, \D) - \obj(\f^*, \D) \leq \frac{8 d^2 \log^2(d/\delta)(\c +
\reg)}{\reg^2 n^2 \priveps^2}.
	\end{align*}
Furthermore, the results of \cite{SridharanSS:07nips} show that with probability $1 - \delta$ over the choice of the data distribution,
	\begin{align*}
	\expobj(\fpriv) - \expobj(\frtr) 
		\leq 
		2 (\obj(\fpriv, \D) - \obj(\f^*,\D)) 
			+ O\left(\frac{\log(1/\delta)}{\reg n}\right).
	\end{align*}
The constant in the last term depends on the derivative of the loss and the
bound on the data points, which by assumption are bounded.  Combining the
preceeding two statements, with probability $1 - 2 \delta$ over the noise in the privacy mechanism and the data distribution, the second term in the right-hand-side of  (\ref{eqn:gen}) is at most:
	\begin{align}
	\expobj(\fpriv) - \expobj(\frtr) 
		\le \frac{16 d^2 \log^2(d/\delta)(\c + \reg)}{\reg^2 n^2 \priveps^2} 
				+ O\left(\frac{\log(1/\delta)}{\reg n}\right).
	\label{eq:ssdiff}
	\end{align}
By definition of $\frtr$, the difference $( \expobj(\frtr) - \expobj(\fopt)) \le 0$.  Setting $\reg = \frac{\geneps}{||\fopt||^2}$ in (\ref{eqn:gen}) and using (\ref{eq:ssdiff}), we obtain
	\begin{align}
	L(\fpriv) \le L(\fopt) + 
		\frac{16 ||\fopt||^4 d^2 \log^2(d/\delta)(\c + \geneps/||\fopt||^2)
		}{
		n^2 \geneps^2 \priveps^2} 
				+ O\left(||\fopt||^2 \frac{\log(1/\delta)}{n \geneps}\right)
		+ \frac{\geneps}{2}.
	\end{align}
Solving for $n$ to make the total excess error equal to $\geneps$ yields (\ref{eq:outputGen}).
\end{proof}

\begin{lemma}
Suppose $N(\cdot)$ is doubly differentiable with $||\nabla^2 N(\f)||_2 \leq \eta$ for all $\f$, and suppose that $\loss$ is differentiable and has continuous and $\c$-Lipschitz derivatives.
Given training data $\D$, let $\f^*$ be a classifier that minimizes $\obj(\f, \D)$ and let $\fpriv$ be the classifier output by Algorithm \ref{alg:output}.  Then
	\begin{align} 
	\prob_{\mbf{b}} \left(
		\obj(\fpriv,\D) 
			\leq \obj(\f^*, \D) 
				+ \frac{2 d^2(\c + \reg \eta) \log^2(d/\delta)}{\reg^2 n^2 \priveps^2} 
		\right) 
	\ge 1 - \delta,
	\end{align}
where the probability is taken over the randomness in the noise $\mbf{b}$ of Algorithm \ref{alg:output}.
\label{lem:trainerrorsens}
\end{lemma}

Note that when $\loss$ is doubly differentiable, $\c$ is an upper bound on
the double derivative of $\loss$, and is the same as the constant $\c$ in
Theorem~\ref{thm:objectivepriv}. 

\begin{proof}
Let $\D = \{ (\x_1, y_1), \ldots, (\x_n, y_n)\}$, and recall that $||\x_i|| \leq 1$, and $|y_i| \leq 1$. 
As $N(\cdot)$ and $\loss$ are differentiable, we use the Mean Value Theorem to show that for some $t$ between $0$ and $1$,
	\begin{align}
 	\obj(\fpriv, \D) - \obj(\f^*, \D)  
		&= (\fpriv - \f^*)^T \nabla \obj(t\f^* + (1 - t)\fpriv) \nonumber \\
		&\leq  ||\fpriv - \f^*|| \cdot ||\nabla \obj(t\f^* + (1 - t) \fpriv)||, 
	\label{eqn:step1}
	\end{align}
where the second step follows by an application of the Cauchy-Schwartz inequality. 
Recall that 
	\begin{align*}
	\nabla \obj(\f, \D) = \reg \nabla N(\f) 
		+ \frac{1}{n} \sum_i y_i \loss'(y_i \f^T \x_i) \x_i.
	\end{align*}
Moreover, recall that $\nabla \obj(\f^*, \D) = 0$, from the optimality of $\f^*$. Therefore,
	\begin{align}
	\nabla \obj(t\f^* + (1 - t)\fpriv, \D) 
	&= \nabla \obj(\f^*, \D) -  \reg (\nabla N(\f^*) 
		- \nabla N(t\f^* + (1 - t)\fpriv)) \nonumber \\
	&\hspace{0.2in} 
		- \frac{1}{n} \sum_i y_i \left(\loss'(y_i (\f^*)^T \x_i) 
		- \loss'(y_i (t \f^* + (1 - t)\fpriv)^T \x_i)\right) \x_i.
	\label{eq:gradconv}
	\end{align}
Now, from the Lipschitz condition on $\loss$, for each $i$ we can upper bound each term in the summation above:
	\begin{align}
	\norm{ y_i \left( \loss'(y_i (\f^*)^T \x_i) - \loss'(y_i (t \f^* + (1 - t)\fpriv)^T \x_i) \right) \x_i }
	& \nonumber \\
	&\hspace{-2in} \le
		|y_i| \cdot ||\x_i|| \cdot |\loss'(y_i (\f^*)^T \x_i) - \loss'(y_i (t \f^* + (1 - t) \fpriv)^T \x_i)| \nonumber \\
	&\hspace{-2in} \leq
		|y_i| \cdot ||\x_i|| \cdot \c \cdot |y_i (1 - t) (\f^* - \fpriv)^T \x_i| \nonumber \\
	&\hspace{-2in} \leq
		\c(1-t)|y_i|^2 \cdot ||\x_i||^2 \cdot ||\f^* - \fpriv|| \nonumber \\
	&\hspace{-2in} \leq
		\c (1 - t) ||\f^* - \fpriv||.
		\label{eq:lossconvnorm}
	\end{align}
The third step follows because $\loss'$ is $\c$-Lipschitz and the last 
step follows from the bounds on $|y_i|$ and $||\x_i||$.
Because $N$ is doubly differentiable, we can apply the Mean Value
Theorem again to conclude that
	\begin{align} 
		||\nabla N(t\f^* + (1-t)\fpriv) - \nabla N(\f^*)|| 
		\leq 
		(1 - t)||\fpriv - \f^*||\cdot ||\nabla^2 N(\f'')||_2
	\label{eq:regconvnorm}
	\end{align}
for some $\f'' \in \bbR^d$.

As $0 \leq t \leq 1$, we can combine (\ref{eq:gradconv}), (\ref{eq:lossconvnorm}), and (\ref{eq:regconvnorm}) to obtain
	\begin{align}
	\norm{\nabla \obj(t\f^* + (1 - t)\fpriv, \D)} 
	& \leq 
		\norm{\reg (\nabla N(\f^*) - \nabla N(t\f^* + (1 - t)\fpriv))} \nonumber \\
	&\hspace{0.2in} +
		\norm{\frac{1}{n} \sum_i y_i (\loss'(y_i (\f^*)^T \x_i) 
			- \loss'(y_i (t \f^* + (1 - t)\fpriv)^T \x_i)) \x_i} \nonumber \\
	&\le 
		(1-t) \norm{\fpriv - \f^*} \cdot \left( \reg \eta + \frac{1}{n} \cdot n \cdot \c \right) \nonumber \\
	&\le
		\norm{\fpriv - \f^*} (\reg \eta + \c).
	\label{eqn:step2}
\end{align}

From the definition of Algorithm \ref{alg:output}, $\fpriv - \f^* = \b$, where $\b$ is the noise vector.  Now we can apply Lemma \ref{lem:gammaproperty} to $||\fpriv - \f^*||$, with parameters $k = d$, and $\theta = \frac{2}{\reg n \priveps}$. From Lemma~\ref{lem:gammaproperty}, with
probability $1 - \delta$, $||\fpriv - \f^*|| \leq \frac{2 d\log(\frac{d}{\delta})}{\reg n \priveps}$. The Lemma follows by 
combining this with Equations~\ref{eqn:step2} and \ref{eqn:step1}.
\end{proof}

\begin{lemma}
Let $X$ be a random variable drawn from the distribution $\Gamma(k, \theta)$, where $k$ is an integer. Then,
	\begin{align}
	\prob\left( X < k \theta \log\left(\frac{k}{\delta}\right) \right) \geq 1- \delta.			\end{align}
\label{lem:gammaproperty}
\end{lemma}
	
\begin{proof}
Since $k$ is an integer, we can decompose $X$ distributed according to $\Gamma(k,\theta)$ as a summation
	\begin{align}
	X = X_1 + \ldots + X_k,
	\end{align}
where $X_1, X_2, \ldots, X_k$ are independent exponential random variables with mean $\theta$.  For each $i$ we have $\prob( X_i \ge \theta \log(k/\delta) ) = \delta/k$.  Now,
	\begin{align}
	\prob( X < k \theta \log(k/\delta) ) 
	&\ge \prob( X_i < \theta \log(k/\delta)\ i = 1, 2, \ldots, k) \\
	&= (1 - \delta/k)^k \\
	&\ge 1 - \delta.
	\end{align}
\end{proof}

\subsection{Objective perturbation}

We now establish performance bounds on Algorithm \ref{alg:objective}. The bound can be summarized as follows.

\begin{theorem}
Let $N(\f) = \frac{1}{2}||\f||^2$, and let $\fopt$ be a classifier with
expected loss $L(\fopt) = \bestloss$.  Let $\loss$ be convex, doubly
differentiable, and let its derivatives satisfy $|\loss'(z)| \leq 1$ and
$|\loss''(z)|\leq \c$ for all $z$.  Then there exists a constant $C$ such
that for $\delta > 0$, if the $n$ training samples in $\D$ are drawn i.i.d.
according to $P$, and if 
	\begin{align} 
	n > C \max\left(
		\frac{||\fopt||^2 \log(1/\delta)}{\geneps^2},
		\frac{\c||\fopt||^2}{\geneps \priveps}, 
		\frac{d\log(\frac{d}{\delta}) ||\fopt||}{\geneps \priveps}
		\right),
	\label{eq:objSampCmplx}
	\end{align}
then the output $\fpriv$ of Algorithm \ref{alg:objective} satisfies
	\begin{align}
	\prob\left( L(\fpriv) \le \bestloss + \geneps \right) \ge 1 - 2 \delta.
	\end{align}
\label{thm:objaccuracy}
\end{theorem}

\begin{proof}
Let 
	\begin{align*}
	\frtr &= \argmin_{\f} \expobj(\f) \\
	\f^{\ast} &= \argmin_{\f} \obj(\f,\D),
	\end{align*}
and $\fpriv$ denote the output of Algorithm \ref{alg:output}.  As in Theorem \ref{thm:sensaccuracy}, the analysis of \cite{ShalevShwartzS:08icml} shows
	\begin{align}
	L(\fpriv) &=  L(\fopt) + ( \expobj(\fpriv) - \expobj(\frtr)) 
	   + ( \expobj(\frtr) - \expobj(\fopt)) + \frac{\reg}{2} ||\fopt||^2 - \frac{\reg}{2} ||\fpriv||^2.
	\label{eqn:gen2}
	\end{align}
We will bound each of the terms on the right-hand-side.

If $n > \frac{\c ||\fopt||^2}{\geneps \priveps}$ and $\reg > \frac{\geneps}{4 ||\fopt||^2}$, then $n \reg > \frac{\c}{4 \priveps}$, so from the definition of $\priveps'$ in Algorithm \ref{alg:objective}, 
	\begin{align} 
	\priveps' = \priveps - 2 \log\left(1 + \frac{\c}{n \reg}\right) 
	= \priveps - 2 \log\left(1 + \frac{\priveps}{4}\right) 
	\geq \priveps - \frac{\priveps}{2},
	\end{align}
where the last step follows because $\log(1 + x) \le x$ for $x \in [0,1]$. Note that for these values of $\reg$ we have $\priveps' > 0$.

Therefore, we can apply Lemma \ref{lem:trainerrorobj} to conclude that with
probability at least $1 - \delta$ over the privacy mechanism,
	\begin{align}
	\obj(\fpriv, \D) - \obj(\f^*, \D) \leq \frac{4 d^2 \log^2(d/\delta)}{\reg
n^2 \priveps^2}.
	\end{align}
From \cite{SridharanSS:07nips}, 
	\begin{align} 
	\expobj(\fpriv) - \expobj(\frtr) 
	&\leq 2 (\obj(\fpriv, \D) - \obj(\f^*, \D)) 
		+ O\left(\frac{\log(1/\delta)}{\reg n}\right) \\
	&\le \frac{8 d^2 \log^2(d/\delta)}{\reg n^2 \priveps^2} +
O\left(\frac{\log(1/\delta)}{\reg n}\right).
	\end{align}

By definition of $\f^*$, we have $\expobj(\frtr) - \expobj(\fopt) \le 0$.
If $\reg$ is set to be $\frac{\geneps}{||\fopt||^2}$, then, the fourth
quantity in Equation \ref{eqn:gen2} is at most $\frac{\geneps}{2}$. The
theorem follows by solving for $n$ to make the total excess error at most
$\geneps$.
\end{proof}

The following lemma is analogous to Lemma~\ref{lem:trainerrorsens}, and it establishes a bound on the distance between the output of Algorithm \ref{alg:objective}, and non-private regularized ERM. We note that this bound holds when Algorithm \ref{alg:objective} has $\priveps' > 0$, that is, when $\extra = 0$. Ensuring that $\extra = 0$ requires an additional condition on $n$, which is stated in Theorem~\ref{thm:objaccuracy}. 

\begin{lemma}
Let $\priveps' > 0$.  Let $\f^* = \argmin \obj(\f, \D)$, and let $\fpriv$ be the classifier output by Algorithm \ref{alg:objective}.   If
$N(\cdot)$ is $1$-strongly convex and globally differentiable, and if $\loss$ is convex and
differentiable at all points, with $|\loss'(z)|\leq 1$ for all $z$, then
	\begin{align} 
	\prob_{\mbf{b}} \left(
		\obj(\fpriv, \D)  \leq \obj(\f^*, \D) + \frac{4d^2 \log^2(d/\delta)}{\reg n^2 \priveps^2}
		\right) 
	\ge 1 - \delta,
	\end{align}
where the probability is taken over the randomness in the noise $\mbf{b}$ of Algorithm \ref{alg:objective}.
\label{lem:trainerrorobj}
\end{lemma}

\begin{proof}
By the assumption $\priveps' > 0$, the classifier $\fpriv$ minimizes the objective function $\obj(\f, \D) + \frac{1}{n}\b^T \f$, and therefore
	\begin{align} 
	\obj(\fpriv, \D) \leq \obj(\f^*, \D) + \frac{1}{n}\b^T(\f^* - \fpriv).
	\end{align} 
First, we try to bound $||\f^* - \fpriv||$. Recall that $\reg N(\cdot)$ is $\reg$-strongly convex and globally differentiable,
and $\loss$ is convex and
differentiable. We can therefore apply Lemma \ref{lem:sensitivity} with $G(\f) = \obj(\f, \D)$ and $g(\f) = \frac{1}{n} \b^T \f$ to obtain the bound
	\begin{align}
	||\f^* - \fpriv|| \leq \frac{1}{\reg} \norm{\nabla( \frac{1}{n} \b^T \f )} \le \frac{||\b||}{n \reg}.
	\end{align}
Therefore by the Cauchy-Schwartz inequality,
	\begin{align}
	\obj(\fpriv, \D) - \obj(\f^*, \D) \le \frac{||\b||^2}{n^2 \reg}.
	\end{align}
Since $||\b||$ is drawn from a $\Gamma(d, \frac{2}{\priveps})$ distribution, from Lemma \ref{lem:gammaproperty}, with probability $1 - \delta$, $||\b|| \leq \frac{2d\log(d/\delta)}{\priveps}$.  The Lemma follows by plugging this in to the previous equation.
\end{proof}

\subsection{Applications}

In this section, we examine the sample requirement of privacy-preserving
regularized logistic regression and support vector machines. Recall that in
both these cases, $N(\f) = \frac{1}{2} ||\f||^2$.

\begin{corollary}[Logistic Regression]
Let training data $\mc{D}$ be generated i.i.d. according to a distribution $P$ and let
 $\fopt$ be a classifier with expected loss $L(\fopt) = \bestloss$.  Let the loss function 
$\loss = \llr$ defined in Section \ref{sec:PrivLR}.  Then the following two statements hold:
	\begin{enumerate}
	\item There exists a $C_1$ such that if
		\begin{align}
		n > C_1 \max\left(
			\frac{||\fopt||^2 \log(\frac{1}{\delta})}{\geneps^2},
			\frac{d \log (\frac{d}{\delta})||\fopt||}{\geneps \priveps}, 
			\frac{d\log(\frac{d}{\delta}) ||\fopt||^2}{\geneps^{3/2} \priveps}
			\right),
		\end{align}
	then the output $\fpriv$ of Algorithm \ref{alg:output} satisfies
		\begin{align}
		\prob\left( L(\fpriv) \le \bestloss + \geneps \right) \ge 1 - \delta.
		\end{align}
	\item There exists a $C_2$ such that if
		\begin{align}
		n > C \max\left(
			\frac{||\fopt||^2 \log(1/\delta)}{\geneps^2},
			\frac{||\fopt||^2}{\geneps \priveps}, 
			\frac{d\log(\frac{d}{\delta}) ||\fopt||}{\geneps \priveps}
			\right),
		\end{align}
	then the output $\fpriv$ of Algorithm \ref{alg:objective} 
	with $\c=\frac{1}{4}$ satisfies
		\begin{align}
		\prob\left( L(\fpriv) \le \bestloss + \geneps \right) \ge 1 - \delta.
		\end{align}
	\end{enumerate}
\label{cor:accuracylr}
\end{corollary}

\begin{proof}
Since $\llr$ is convex and doubly differentiable for any $z_1$, $z_2$,
	\begin{align} 
	\llr'(z_1) - \llr'(z_2) \leq \llr''(z^*)(z_1 - z_2) 
	\end{align}
for some $z^* \in [z_1, z_2]$.  Moreover, $|\llr''(z^*)| \leq
\c=\frac{1}{4}$, so $\loss'$ is $\frac{1}{4}$-Lipschitz.  The corollary now follows from Theorems~\ref{thm:sensaccuracy} and~\ref{thm:objaccuracy}.
\end{proof}

For SVMs we state results with $\loss=\lhuber$, but a similar bound 
can be shown for $\loss_{s}$ as well.

\begin{corollary}[Huber Support Vector Machines]
Let training data $\mc{D}$ be generated i.i.d. according to a distribution $P$ and let
 $\fopt$ be a classifier with expected loss $L(\fopt) = \bestloss$.  Let the loss function 
$\loss = \lhuber$ defined in \eqref{eq:Huber}. Then the following two statements hold:
\begin{enumerate}
	\item There exists a $C_1$ such that if
		\begin{align}
		n > C_1 \max\left(
			\frac{||\fopt||^2 \log(\frac{1}{\delta})}{\geneps^2},
			\frac{d \log (\frac{d}{\delta})||\fopt||}{\geneps \priveps}, 
			\frac{d\log(\frac{d}{\delta}) ||\fopt||^2}{h^{1/2} \geneps^{3/2} \priveps}
			\right),
		\end{align}
	then the output $\fpriv$ of Algorithm \ref{alg:output} satisfies
		\begin{align}
		\prob\left( L(\fpriv) \le \bestloss + \geneps \right) \ge 1 - \delta.
		\end{align}
	\item There exists a $C_2$ such that if
		\begin{align}
		n > C \max\left(
			\frac{||\fopt||^2 \log(1/\delta)}{\geneps^2},
			\frac{||\fopt||^2}{h \geneps \priveps}, 
			\frac{d\log(\frac{d}{\delta}) ||\fopt||}{\geneps \priveps}
			\right),
		\end{align}
	then the output $\fpriv$ of Algorithm \ref{alg:objective} 
	with $\c=\frac{1}{4}$ satisfies
		\begin{align}
		\prob\left( L(\fpriv) \le \bestloss + \geneps \right) \ge 1 - \delta.
		\end{align}
	\end{enumerate}
\label{cor:accuracysvm}
\end{corollary}

\begin{proof}
The Huber loss is convex and differentiable with continuous derivatives.
Moreover, since the derivative of the Huber loss is piecewise linear with
slope $0$ or at most $\frac{1}{2h}$, for any $z_1$, $z_2$,
	\begin{align}
	|\lhuber'(z_1) - \lhuber'(z_2)| \leq \frac{1}{2h}|z_1 - z_2|,
	\end{align}
so $\lhuber'$ is $\frac{1}{2h}$-Lipschitz.  The first part of the corollary follows from Theorem~\ref{thm:sensaccuracy}. 

For the second part of the corollary, we observe that from 
Corollary~\ref{cor:huber}, we do not need $\loss$ to be
globally double differentiable, and the bound on $|\loss''(z)|$ in
Theorem~\ref{thm:objaccuracy} is only needed to ensure that $\priveps'>0$; 
since $\lhuber$ is double differentiable except in a set of 
Lebesgue measure $0$, with $|\lhuber''(z)| \leq \frac{1}{2h}$, the corollary follows by an application of Theorem~\ref{thm:objaccuracy}.
\end{proof}

\section{Kernel methods  \label{sec:kernels}}

A powerful methodology in learning problems is the ``kernel trick,'' which allows the efficient construction of a predictor $\f$ that lies in a reproducing kernel Hilbert space (RKHS) $\mc{H}$ associated to a positive definite kernel function $k(\cdot,\cdot)$.  The representer theorem \citep{KimerldorfW:70representer} shows that the regularized empirical risk in (\ref{eq:StdRegERM}) is minimized by a function $\f(\mbf{x})$ that is given by a linear combination of kernel functions centered at the data points:
	\begin{align}
	\f^{\ast}(\mbf{x}) = \sum_{i=1}^{n} a_i k(\mbf{x}(i),\mbf{x}).
	\label{eq:optkerclass}
	\end{align}
This elegant result is important for both theoretical and computational reasons.   Computationally, one releases the values $a_i$ corresponding to the $\f$ that minimizes the empirical risk, along with the data points $\mbf{x}(i)$; the user classifies a new $\mbf{x}$ by evaluating the function in \eqref{eq:optkerclass}.

A crucial difficulty in terms of privacy is that this directly releases the private values $\mbf{x}(i)$ of some individuals in the training set. Thus, even if the classifier is  computed in a privacy-preserving way, any classifier released by this process requires revealing the data.  We provide an algorithm that avoids this problem, using an approximation method \citep{RahimiR:07features,RahimiR:08kitchen} to approximate the kernel function using random projections.

\subsection{Mathematical preliminaries}

Our approach works for kernel functions which are translation invariant, so
$k(\mbf{x},\mbf{x}') = k(\mbf{x} - \mbf{x}')$.  The key idea in the random
projection method is from Bochner's Theorem, which states that a continuous
translation invariant kernel is positive definite if and only if it is the Fourier transform of a nonnegative measure.  This means that the Fourier transform $K(\theta)$ of translation-invariant kernel function $k(\mbf{t})$ can be normalized so that $\kft(\theta) = K(\theta)/\norm{K(\theta)}_1$ is a probability measure on the transform space $\Theta$.  We will assume $\kft(\theta)$ is uniformly bounded over $\theta$.  

In this representation
	\begin{align}
	k(\mbf{x},\mbf{x}') = \int_{\Theta} \phi(\mbf{x} ; \theta) \phi(\mbf{x}'; \theta)\kft(\theta) d\theta,
	\label{eq:kernelTheta}
	\end{align}
where we will assume the feature functions $\phi(\mbf{x} ; \theta)$ are bounded:
	\begin{align}
	|\phi(\mbf{x}; \theta)| \le \zeta \qquad \forall \mbf{x} \in \mc{X},\ \forall \theta \in \Theta.
	\end{align}
A function $\f \in \mc{H}$ can be written as
	\begin{align}
	\f(\mbf{x}) = \int_{\Theta} a(\theta) \phi(\mbf{x} ; \theta) \kft(\theta) d\theta.
	\end{align}
To prove our generalization bounds we must show that bounded classifiers $\f$ induce bounded functions $a(\theta)$.  Writing the evaluation functional as an inner product with $k(\mbf{x},\mbf{x'})$ and (\ref{eq:kernelTheta}) shows
	\begin{align}  
	\f(\mbf{x}) 
	&= \int_{\Theta} \left( \int_{\mc{X}} \f(\mbf{x}') 
				\phi(\mbf{x}' ; \theta) d\mbf{x}' \right)
		\phi(\mbf{x} ; \theta) \kft(\theta) d\theta.
	\end{align}
Thus we have
	\begin{align}
	a(\theta) &= \int_{\mc{X}} \f(\mbf{x}') \phi(\mbf{x}' ; \theta) d\mbf{x}' \\ 
	|a(\theta)| &\le \Vol(\mc{X}) \cdot \zeta \cdot \norm{\f}_{\infty}.  \label{eq:ballbound}
	\end{align}
This shows that $a(\theta)$ is bounded uniformly over $\Theta$ when $\f(\mbf{x})$ is bounded uniformly over $\mc{X}$.  The volume of the unit ball is $\Vol(\mc{X}) = \frac{\pi^{d/2}}{ \Gamma(\frac{d}{2} + 1) }$ (see \cite{Ball:97convex} for more details).  For large $d$ this is $(\sqrt{\frac{2 \pi e}{d}})^d$ by Stirling's formula.  Furthermore, we have
	\begin{align}
	\norm{\f}_{\mc{H}}^2 = \int_{\Theta} a(\theta)^2 \kft(\theta) d\theta.
	\end{align}

\subsection{A reduction to the linear case}

We now describe how to apply Algorithms 1 and 2 for classification with kernels, by transforming to linear classification.
Given $\{\theta_j\}$, let $R : \mc{X} \to \bbR^D$ be the map that sends $\mbf{x}(i)$ to a vector $\mbf{v}(i) \in \bbR^D$ where $\mbf{v}_j(i) = \phi(\mbf{x}(i); \theta_j)$
for $j \in [D]$.  We then use Algorithm 1 or Algorithm 2 to compute a privacy-preserving linear classifier $\f$ in $\bbR^D$.  The algorithm releases $R$ and $\tilde{\f}$.  The overall classifier is $\fpriv(\mbf{x}) = \tilde{\f}(R(\mbf{x}))$.

\begin{algorithm}
\caption{Private ERM for nonlinear kernels}
\label{alg:kernel}
\begin{algorithmic}
\STATE \textbf{Inputs:} Data $\{(\mbf{x}_i, y_i) : i \in [n]\}$, positive definite kernel function $k(\cdot, \cdot)$, sampling function $\kft(\theta)$, parameters $\priveps$, $\reg$, $D$
\STATE \textbf{Outputs:} Predictor $\fpriv$ and pre-filter $\{ \theta_j : j \in [D] \}$.
\STATE Draw $\{\theta_j : j = 1, 2, \ldots, D\}$ iid according to $\kft(\theta)$.
\STATE Set $\mbf{v}(i) = \sqrt{2/D} [\phi(\mbf{x}(i); \theta_1) \cdots \phi(\mbf{x}(i); \theta_D) ]^T$ for each $i$.
\STATE Run Algorithm 1 or Algorithm 2 with data $\{(\mbf{v}(i),y(i))\}$ and parameters $\priveps$, $\reg$.
\end{algorithmic}
\end{algorithm}

As an example, consider the Gaussian kernel
	\begin{align}
	k(\mbf{x},\mbf{x}') = \exp\left( - \gamma \norm{ \mbf{x} - \mbf{x}' }_2^2 \right).
	\end{align}
The Fourier transform of a Gaussian is a Gaussian, so we can sample $\theta_j = (\omega,\psi)$ according to the distribution $\mathrm{Uniform}[-\pi,\pi] \times \mc{N}(0,2 \gamma I_d)$ and compute $v_j = \cos( \omega^T \mbf{x} + \psi )$.  The random phase is used to produce a real-valued mapping.  The paper of \cite{RahimiR:08allerton} has more examples of transforms for other kernel functions.

\subsection{Privacy guarantees}

Because the workhorse of Algorithm \ref{alg:kernel} is a differentially-private version of ERM for linear classifiers (either  Algorithm \ref{alg:output} or Algorithm \ref{alg:objective}),  and the points $\{ \theta_j : j \in [D] \}$ are independent of the data, the privacy guarantees for Algorithm \ref{alg:kernel} follow trivially from Theorems \ref{thm:sensitivitypriv} and \ref{thm:objectivepriv}.

\begin{theorem}
\label{thm:ker:priv}
Given data $\{(\mbf{x}(i), y(i)) : i = 1, 2, \ldots, n\}$ with $(\mbf{x}(i), y(i))$ and $\norm{\mbf{x}(i)} \le 1$, the outputs $(\fpriv,\{ \theta_j : j \in [D] \} )$ of Algorithm \ref{alg:kernel} guarantee $\priveps$-differential privacy.
\end{theorem}

The proof trivially follows from a combination of Theorems \ref{thm:sensitivitypriv}, \ref{thm:objectivepriv}, and the fact that the $\theta_j$'s are drawn independently of the input dataset.

\subsection{Generalization performance}\label{sec:kernelBounds}

We now turn to generalization bounds for Algorithm \ref{alg:kernel}.  We will prove results using objective perturbation (Algorithm \ref{alg:objective}) in Algorithm \ref{alg:kernel}, but analogous results for output perturbation (Algorithm \ref{alg:output}) are simple to prove.  Our comparisons will be against arbitrary predictors $\fopt$ whose norm is bounded in some sense.  That is, given an $\fopt$ with some properties, we will choose regularization parameter $\reg$, dimension $D$, and number of samples $n$ so that the predictor $\fpriv$ has expected loss close to that of $\fopt$.  %

In this section we will assume $N(\f) = \frac{1}{2} \norm{\f}^2$ so that $N(\cdot)$ is $1$-strongly convex, and that the loss function $\loss$ is convex, differentiable and $|\loss'(z)| \le 1$ for all $z$.

Our first generalization result is the simplest, since it assumes a strong
condition that gives easy guarantees on the projections.  We would like the
predictor produced by Algorithm~\ref{alg:kernel} to be competitive against an $\fopt$ such that
	\begin{align}
	\fopt(\mbf{x}) = \int_{\Theta} \aopt(\theta) \phi(\mbf{x};\theta) \kft(\theta) d\theta,
	\label{eq:aopt}
	\end{align}
and $|\aopt(\theta)| \le C$ (see \cite{RahimiR:08kitchen}).  Our first result provides the technical building block for our other generalization results.  The proof makes use of ideas from \cite{RahimiR:08kitchen} and techniques from \cite{SridharanSS:07nips,ShalevShwartzS:08icml}.

\begin{lemma}
\label{lem:ker:thetagen}
Let $\fopt$ be a predictor such that $|\aopt(\theta)| \le C$, for all
$\theta$, where $\aopt(\theta)$ is given by (\ref{eq:aopt}), and suppose
$\exploss(\fopt) = \bestloss$.  Moreover, suppose that $\loss'(\cdot)$ is
$\c$-Lipschitz. If the data $\D$ is drawn i.i.d. according
to $P$, then there exists a constant $C_0$ such that if
	\begin{align}
	n > C_0 \cdot  \max \left(\frac{C^2 \sqrt{\log(1/\delta)}}{
	\priveps \geneps^2 } \cdot \log \frac{C \log (1/\delta)}{\geneps
	\delta}, \frac{\c \geneps}{\priveps
\log(1/\delta)} \right) ,
		\label{eq:inftysamples}
	\end{align}
then $\reg$ and $D$ can be chosen such that the output $\fpriv$ of Algorithm \ref{alg:kernel} using Algorithm \ref{alg:objective} satisfies 
	\begin{align}
	\prob\left( \exploss(\fpriv) - \bestloss \le \geneps \right) \ge 1 - 4 \delta.
	\end{align}
\end{lemma}

\begin{proof}
Since $|\aopt(\theta)| \le C$ and the $\kft(\theta)$ is bounded, we have \cite[Theorem 1]{RahimiR:08kitchen} that with probability $1 - 2 \delta$ there exists an $\f_p \in \mathbb{R}^D$ such that
	\begin{align}
	\exploss(\f_p) \le \exploss(\fopt) 
		+ O\left( \left( \frac{1}{\sqrt{n}} + \frac{1}{\sqrt{D}} \right) 
				C \sqrt{ \log \frac{1}{\delta} } \right),
	\label{eq:RRloss}
	\end{align}
We will choose $D$  to make this loss small.  Furthermore, $\f_p$ is guaranteed to have $\norm{\f_p}_{\infty} \le C/D$, so
	\begin{align}
	\norm{\f_p}_2^2 \le \frac{C^2}{D}.
	\label{eq:f0norm}
	\end{align}
	
Now given such an $\f_p$ we must show that $\fpriv$ will have true risk close to that of $\f_p$ as long as there are enough data points.  This can be shown using the techniques in \cite{ShalevShwartzS:08icml}.  Let
	\begin{align*}
	\expobj(\f) &= \exploss(\f) + \frac{\reg}{2} \norm{\f}_2^2, 
	\end{align*}
and let 
	\begin{align*}
	\frtr &= \argmin_{\f \in \mathbb{R}^D} \expobj(\f)
	\end{align*}
minimize the regularized true risk.  Then
	\begin{align*}
	\expobj(\fpriv) = \expobj(\f_p) + ( \expobj(\fpriv) - \expobj(\frtr) )
		+ ( \expobj(\frtr) - \expobj( \f_p ) ).
	\end{align*}
Now, since $\expobj(\cdot)$ is minimized by $\frtr$, the last term is negative and we can disregard it.  Then we have
	\begin{align}
	\exploss(\fpriv) - \exploss(\f_p) \le ( \expobj(\fpriv) - \expobj(\frtr) ) 
		+ \frac{\reg}{2} \norm{\f_p}_2^2 - \frac{\reg}{2} \norm{\fpriv}_2^2.
	\label{eq:explossdiff}
	\end{align}
From Lemma \ref{lem:trainerrorobj}, with probability at least $1 - \delta$ over the noise $\b$, 
	\begin{align}
	\obj(\fpriv) - \obj\left( \argmin_{\f} \obj(\f) \right) 
	\le \frac{ 4 D^2 \log^2(D/\delta) }{ \reg n^2 \priveps^2 }.
	\label{eq:empdiff}
	\end{align}
Now using \cite[Corollary 2]{SridharanSS:07nips}, we can bound the term $(\expobj(\fpriv) - \expobj(\femp))$ by twice the gap in the regularized empirical risk difference (\ref{eq:empdiff}) plus an additional term.  That is, with probability $1 - \delta$:
	\begin{align}
	\expobj(\fpriv) - \expobj(\frtr)
	&\le
	2 (\obj(\fpriv) - \obj(\frtr)) + O\left( \frac{\log(1/\delta)}{ \reg n } \right).
	\label{eq:explossboundprev}
	\end{align}
If we set $n > \frac{\c}{4 \priveps \reg}$, then $\priveps' > 0$, and we
can plug Lemma~\ref{lem:trainerrorobj} into
\eqref{eq:explossboundprev} to obtain:
	\begin{align}
	\expobj(\fpriv) - \expobj(\frtr)
	&\le 
	\frac{ 8 D^2 \log^2(D/\delta) }{ \reg n^2 \priveps^2 }
	+ O\left( \frac{\log(1/\delta)}{ \reg n } \right).
	\label{eq:explossbound}
	\end{align}
Plugging (\ref{eq:explossbound}) into (\ref{eq:explossdiff}), discarding the negative term involving $\norm{\fpriv}_2^2$ and setting $\reg = \geneps/\norm{\f_p}^2$ gives
	\begin{align}
	\exploss(\fpriv) - \exploss(\f_p)
	&\le
	\frac{ 8 \norm{\f_p}_2^2 D^2 \log^2(D/\delta) }{ n^2 \priveps^2 \geneps }
	+ O\left( \frac{ \norm{\f_p}_2^2 \log \frac{1}{\delta}}{ n \geneps } \right)
	+ \frac{\geneps}{2}.
	\label{eq:privloss}
	\end{align}

Now we have, using (\ref{eq:RRloss}) and (\ref{eq:privloss}), that with probability $1 - 4 \delta$:
	\begin{align*}
	\exploss(\fpriv) - \exploss(\fopt)
	&\le
	(\exploss(\fpriv) - \exploss(\f_p)) + (\exploss(\f_p) - \exploss(\fopt)) \\
	&\le \frac{ 8 \norm{\f_p}_2^2 D^2 \log^2(D/\delta) }{ n^2 \priveps^2 \geneps }
	+ O\left( \frac{ \norm{\f_p}_2^2 \log(1/\delta)}{ n \geneps } \right)
	+ \frac{\geneps}{2} \nonumber\\
	& \qquad \qquad \qquad \qquad
	+ O\left( \left( \frac{1}{\sqrt{n}} + \frac{1}{\sqrt{D}} \right) C \sqrt{ \log \frac{1}{\delta} } \right),	
	\end{align*}
Substituting (\ref{eq:f0norm}), we have
	\begin{align*}
	\exploss(\fpriv) - \exploss(\fopt)
	&\le
	\frac{ 8 C^2 D \log^2(D/\delta) }{ n^2 \priveps^2 \geneps }
	+ O\left( \frac{ C^2 \log(1/\delta)}{ D n \geneps } \right)
	+ \frac{\geneps}{2} \\
	&\qquad\qquad
	+ O\left( \left( \frac{1}{\sqrt{n}} + \frac{1}{\sqrt{D}} \right) C \sqrt{ \log \frac{1}{\delta} } \right).
	\end{align*}
To set the remaining parameters, we will choose $D < n$ so that 
	\begin{align}
	\exploss(\fpriv) - \exploss(\fopt)
	&\le
	\frac{ 8 C^2 D \log^2(D/\delta) }{ n^2 \priveps^2 \geneps }
	+ O\left( \frac{ C^2 \log(1/\delta) }{ D n \geneps } \right)
	+ \frac{\geneps}{2}
	+ O\left( \frac{C \sqrt{ \log(1/\delta) }}{\sqrt{D}} \right). \nonumber
	\end{align}
We set $D = O(C^2 \log(1/\delta) / \geneps^2)$ to make the last term $\geneps/6$, and:
	\begin{align*}
	\exploss(\fpriv) - \exploss(\fopt)
	&\le
	O\left( \frac{ C^4 \log \frac{1}{\delta} \log^2 \frac{C^2 \log(1/\delta)}{\geneps^2 \delta} 
		}{ n^2 \priveps^2 \geneps^3 } \right)
	+ O\left( \frac{\geneps}{n} \right)
	+  \frac{2 \geneps}{3}.
	\end{align*}
Setting $n$ as in (\ref{eq:inftysamples}) proves the result. Moreover,
setting  $n > \frac{\c
\norm{\f_p}^2}{4\priveps \geneps} = C_0 \cdot \frac{\c \geneps}{\priveps
\log(1/\delta)}$ ensures that $n > \frac{\c}{4\reg \priveps}$.
\end{proof}

We can adapt the proof procedure to show that Algorithm \ref{alg:kernel} is competitive against any classifier $\fopt$ with a given bound on $\norm{\fopt}_{\infty}$.  It can be shown that for some constant $\zeta$ that $|\aopt(\theta)| \le \Vol(\mc{X}) \zeta\norm{\fopt}_{\infty}$.  Then we can set this as $C$ in (\ref{eq:inftysamples}) to obtain the following result.

\begin{theorem}
\label{thm:ker:inftygen}
Let $\fopt$ be a classifier with norm $\norm{\fopt}_{\infty}$, and let
$\loss'(\cdot)$ be $\c$-Lipschitz. Then for any distribution $P$, there
exists a constant  $C_0$ such that if
	\begin{align}
	n > C_0 \cdot \max \left( \frac{ \norm{\fopt}_{\infty}^2 \zeta^2
	(\Vol(\mc{X}))^2\sqrt{\log(1/\delta)}}{  \priveps \geneps^2 } \cdot \log \frac{ \norm{\fopt}_{\infty} \Vol(\mc{X}) \zeta \log (1/\delta)
			}{\geneps \delta \Gamma(\frac{d}{2} + 1) },
			\frac{\c \geneps}{\priveps \log(1/\delta)} \right),
	\label{eq:kerinftysamp}
	\end{align}
then $\reg$ and $D$ can be chosen such that the output $\fpriv$ of
Algorithm~\ref{alg:kernel} with Algorithm 2 satisfies $\prob\left( \exploss(\fpriv) - \exploss(\fopt) \le \geneps \right) \ge 1 - 4 \delta$.
\end{theorem}

\begin{proof}
Substituting $C = \Vol(\mc{X}) \zeta\norm{\fopt}_{\infty}$ in Lemma
~\ref{lem:ker:thetagen} we get the result.
\end{proof}

We can also derive a generalization result with respect to classifiers with bounded $\norm{\fopt}_{\mc{H}}$. 

\begin{theorem}
\label{thm:ker:l2gen}
Let $\fopt$ be a classifier with norm $\norm{\fopt}_{\mc{H}}$, and let
$\loss'$ be $\c$-Lipschitz. Then for
any distribution $P$, there exists a constant $C_0$ such that if,
	\begin{align}
	n &= C_0 \cdot \max \left( \frac{ \norm{\fopt}_{\mc{H}}^4 \zeta^2
	(\Vol(\mc{X}))^2\sqrt{\log(1/\delta)}}{  \priveps \geneps^4 } \cdot \log \frac{ \norm{\fopt}_{\mc{H}} \Vol(\mc{X}) \zeta \log (1/\delta)
			}{\geneps \delta \Gamma(\frac{d}{2} + 1) }, 	\frac{\c \geneps}{\priveps \log(1/\delta)} \right),
	\end{align}
then $\reg$ and $D$ can be chosen such that the output of
Algorithm~\ref{alg:kernel} run with Algorithm 2 satisfies $\prob\left( \exploss(\fpriv) - \exploss(\fopt) \le \geneps \right) \ge 1 - 4 \delta$.
\end{theorem}

\begin{proof}
Let $\fopt$ be a classifier with norm $\Hnorm{\fopt}^2$ and expected loss $\exploss(\fopt)$.  Now consider
	\begin{align*}
	\frtr = \argmin_{\f} \exploss(\f) + \frac{\regrtr}{2}
	\Hnorm{\f}^2,
	\end{align*}
for some $\regrtr$ to be specified later. We will first need a bound on $\norm{\frtr}_{\infty}$ in order to use our previous sample complexity results.  %
Since $\frtr$ is a minimizer, we can take the derivative of the regularized expected loss and set it to $0$ to get:
	\begin{align*}
	\frtr(\mbf{x}') 
	&= \frac{- 1}{\regrtr} \left(
		\frac{\partial}{\partial \f} \int_{\mc{X}} \ell(\f(\mbf{x}'), y) dP(\mbf{x},y)
		\right) \\
	&= \frac{- 1}{\regrtr} \left(
		\int_{\mc{X}} \left(\frac{\partial}{\partial \f(\mbf{x}')} \ell(\f(\mbf{x}), y)
			\right) \cdot
			\left( \frac{\partial}{\partial \f(\mbf{x}')} \f(\mbf{x}) \right) 
			dP(\mbf{x},y)
			\right),
	\end{align*}
where $P(\mbf{x},y)$ is a distribution on pairs $(\mbf{x},y)$.  Now, using the representer theorem, $\frac{\partial}{\partial \f(\mbf{x}')} \f(\mbf{x}) = k(\mbf{x}',\mbf{x})$. Since the kernel function is bounded and the derivative of the loss is always upper bounded by $1$, so the integrand can be upper bounded by a constant.  Since $P(\mbf{x},y)$ is a probability distribution, we have for all $\mbf{x}'$ that $|\frtr(\mbf{x}')| = O(1/\regrtr)$.  Now we set $\regrtr = \geneps / \Hnorm{\fopt}^2$ to get
	\begin{align*}
	\norm{ \frtr }_{\infty} &= O\left( \frac{\Hnorm{\fopt}^2}{\geneps} \right).
	\end{align*}

We now have two cases to consider, depending on whether $\exploss(\fopt) < \exploss(\frtr)$ or $\exploss(\fopt) > \exploss(\frtr)$.

\textbf{Case 1: }   Suppose that $\exploss(\fopt) < \exploss(\frtr)$.  Then by the definition of $\frtr$, 
	\begin{align*}
	\exploss(\frtr) + \frac{\geneps}{2} \cdot \frac{\Hnorm{\frtr}^2}{\Hnorm{\fopt}^2}
		\le \exploss(\fopt) + \frac{\geneps}{2}.
	\end{align*}
Since $\frac{\geneps}{2} \cdot \frac{\Hnorm{\frtr}^2}{\Hnorm{\fopt}^2} \ge 0$, we have $\exploss(\frtr) - \exploss(\fopt) \le \frac{\geneps}{2}$.

\textbf{Case 2: }   Suppose that $\exploss(\fopt) > \exploss(\frtr)$.  Then the regularized classifier has better generalization performance than the original, so we have trivially that $\exploss(\frtr) - \exploss(\fopt) \le \frac{\geneps}{2}$.

Therefore in both cases we have a bound on $\norm{ \frtr }_{\infty}$ and a generalization gap of $\geneps/2$.  We can now apply Theorem \ref{thm:ker:inftygen} to show that for $n$ satisfying \eqref{eq:kerinftysamp}
we have 
	\begin{align*}
	\prob\left( \exploss(\fpriv) - \exploss(\fopt) \le \geneps \right) \ge 1 - 4 \delta.
	\end{align*}
\end{proof}

\section{Parameter tuning \label{sec:tuning}}

The privacy-preserving learning algorithms presented so far in this
paper assume that the regularization constant $\reg$ is provided as an input,
and is independent of the data. In actual applications of ERM, 
$\reg$ is selected based on the data itself. In this section, 
we address this issue: how to design an ERM algorithm with end-to-end
privacy, which selects $\reg$ based on the data itself.

Our solution is to present a privacy-preserving parameter tuning technique
that is applicable in general machine learning algorithms, beyond ERM.  In practice, one typically tunes  parameters (such as the regularization parameter $\reg$) as follows: using data held out for validation, train predictors $\mbf{f}(\cdot; \reg)$ for multiple values of $\reg$, and select the one which provides the best empirical performance.  However, even though the output of an algorithm preserves $\priveps$-differential privacy for a fixed $\reg$ (as is the case with Algorithms \ref{alg:output} and \ref{alg:objective}), by choosing a $\reg$ based on empirical performance on a validation set may violate $\priveps$-differential privacy guarantees.  That is, if the procedure that picks $\reg$ is not private, then an adversary may use the released classifier to infer the value of $\reg$ and therefore something about the values in the database.

We suggest two ways of resolving this issue. First, if we have access to a smaller publicly
available data from the same distribution, then we can use this as a holdout set to tune $\reg$.  This $\reg$ can be subsequently used to train a classifier on the private data.  Since the value of $\reg$ does not depend on the values in the private data set, this procedure will still preserve the privacy of individuals in the private data.

If no such public data is available, then we need a differentially private tuning procedure. We provide such a procedure below.  The main idea is to train for different values of $\reg$ on separate subsets of the training dataset, so that the total training procedure still maintains $\priveps$-differential privacy.  We score each of these predictors on a validation set, and choose a $\reg$ (and hence $\mbf{f}(\cdot; \reg)$) using a randomized privacy-preserving comparison procedure \citep{MT07}.  The last step is needed to guarantee $\priveps$-differential privacy for individuals in the validation set.  This final algorithm provides an end-to-end guarantee of differential privacy, and renders our privacy-preserving ERM procedure complete. We observe that both these procedures can be used for tuning multiple parameters as well.

\subsection{Tuning algorithm}

\begin{algorithm}
\caption{Privacy-preserving parameter tuning}
\label{alg:pptune}
\begin{algorithmic}
\STATE \textbf{Inputs:} Database $\mc{D}$, parameters $\{\reg_1, \ldots, \reg_m\}$, $\priveps$.
\STATE \textbf{Outputs:} Parameter $\fpriv$.
\STATE Divide $\mc{D}$ into $m+1$ equal portions $\mc{D}_1, \ldots, \mc{D}_{m+1}$, each of size $\frac{\card{\mc{D}}}{m+1}$.
\STATE For each $i = 1, 2, \ldots, m$, apply a privacy-preserving learning
algorithm (e.g. Algorithms \ref{alg:output}, \ref{alg:objective}, or \ref{alg:kernel}) on $\mc{D}_i$ with parameter $\reg_i$ and $\priveps$ to get output $\mbf{f}_i$.
\STATE Evaluate $z_i$, the number of mistakes made by $\mbf{f}_i$ on
$\mc{D}_{m+1}$. Set $\fpriv = \mbf{f}_i$ with probability 
	\begin{align}
	q_i = \frac{e^{-\priveps z_i/2}}{\sum_{i=1}^{m} e^{-\priveps z_i /2}}.
	\end{align}
\end{algorithmic}
\end{algorithm}

We note that the list of potential $\reg$ values input to this procedure should not be a function of the private dataset.  It can be shown that the empirical error on $\mc{D}_{m+1}$ of the classifier output by this procedure is close to the empirical error of the best classifier in the set $\{ \mbf{f}_1, \ldots, \mbf{f}_m \}$ on $\mc{D}_{m+1}$, provided $\card{\mc{D}}$ is high enough.  %

\subsection{Privacy and utility}

\begin{theorem} \label{thm:tune:priv}
The output of the tuning procedure of Algorithm \ref{alg:pptune} is $\priveps$-differentially private.
\end{theorem}
\begin{proof}To show that Algorithm 4 preserves $\priveps$-differential privacy, we
first consider an alternative procedure $\M$. Let $\M$ be the procedure
that releases the values $(\f_1, \ldots, \f_m, i)$ where, $\f_1, \ldots,
\f_m$ are the intermediate values computed in the
second step of Algorithm 4, and $i$ is the index selected by the
exponential mechanism step. We first show that $\M$ preserves
$\priveps$-differential privacy.

Let $\D$ and $\D'$ be two datasets that differ in the value of one
individual such that $\D = \bar{\D} \cup \{ (\x, y) \}$, and $\D' =
\bar{\D} \cup \{ (\x', y') \}$. 

Recall that the datasets $\D_1, \ldots, \D_{m+1}$ are disjoint; moreover,
the randomness in the privacy mechanisms are independent. Therefore, 
	\begin{align}
	\prob\left(\f_1 \in \SS_1, \ldots, \f_m \in \SS_m, i = i^{\ast} | \D\right) &  \nonumber \\
	&\hspace{-1in} 
	= \int_{\SS_1 \times \ldots \SS_m} 
		\prob\left(i = i^{\ast} | \f_1, \ldots, \f_m, \D_{m+1}\right) 
		\mu(\f_1, \ldots, \f_m | \D) 
		d \f_1 \ldots d \f_m \nonumber \\
	&\hspace{-1in} 
	= \int_{\SS_1 \times \ldots \SS_m} 
		\prob\left(i = i^{\ast} | \f_1, \ldots, \f_m, \D_{m+1}\right) 
		\prod_{j=1}^m \mu_j(\f_j | \D_j) 
		d\f_1 \ldots d\f_m, 
	\label{eqn:tuningprod}
	\end{align}
where $\mu_j(\f)$ is the density at $\f$ induced by the classifier run with
parameter $\reg_j$, and $\mu(\f_1, \ldots, \f_m)$ is the joint density at
$\f_1, \ldots, \f_m$, induced by $\M$. Now suppose that $(\x, y) \in \D_j$, 
for $j=m+1$. Then, $\D_k = \D'_k$, and $\mu_j(\f_j | \D_j) = \mu_j(\f_j | \D'_j)$,
for $k \in [m]$. Moreover, given any fixed set $\f_1, \ldots, \f_m$, 
	\begin{align}
	\prob\left(i = i^{\ast} | \D'_{m+1}, \f_1, \ldots, \f_m\right) 
	\leq 
	e^{\priveps} \prob\left(i = i^{\ast} | \D_{m+1}, \f_1, \ldots, \f_m\right).
	\label{eq:tunediff}
	\end{align}
Instead, if $(\x, y) \in \D_j$, for $j \in [m]$, then, $\D_k = \D'_k$, for
$k \in [m+1], k \neq j$. Thus, for a fixed $\f_1, \ldots, \f_m$,
	\begin{align} 
	\prob\left(i = i^* | \D'_{m+1}, \f_1, \ldots, \f_m\right)
		&= \prob\left(i = i^{\ast} | \D_{m+1}, \f_1, \ldots, \f_m\right) \\
	\mu_k(\f_k |\D_k) &\leq e^{\priveps} \mu_k(\f_k | \D'_k).
	\label{eq:tunesame}
	\end{align}
The lemma follows by combining (\ref{eqn:tuningprod})-(\ref{eq:tunesame}).

Now, an adversary who has access to the output of $\M$ can compute the
output of Algorithm 4 itself, without any further access to the dataset.
Therefore, by a simulatibility argument, as
in~\cite{DworkMNS:06sensitivity}, Algorithm 4 also preserves
$\priveps$-differential privacy.
\end{proof}

In the theorem above, we assume that the individual algorithms for
privacy-preserving classification satisfy Definition~\ref{def:densitypriv};
a similar theorem can also be shown when they satisfy a guarantee as in
Corollary~\ref{cor:huber}.

The following theorem shows that the empirical error on $\D_{K+1}$ of the classifier output by the tuning procedure is close to the empirical error of the best classifier in the set $\{ \mbf{f}_1, \ldots, \mbf{f}_K \}$.  The proof of this Theorem follows from Lemma 7 of \cite{MT07}.

\begin{theorem}    \label{thm:tune:err}
Let $z_{\min} = \min_i z_i$, and let $z$ be the number of mistakes made on
$D_{m+1}$ by the classifier output by our tuning procedure. Then, with probability $1 - \delta$,
	\begin{align}
	z \le z_{\min} + \frac{2\log(m/\delta)}{\priveps}.
	\end{align}
\end{theorem}

\begin{proof}
In the notation of \cite{MT07}, the $z_{\min} = OPT$, the base measure $\mu$ is uniform on $[m]$, and
$S_t = \{ i: z_i < z_{\min} + t \}$.  Their Lemma 7 shows that
	\begin{align}
	\prob\left( \bar{S}_{2t} \right) \le \frac{\exp(-\priveps t)}{\mu(S_t)},
	\end{align}
where $\mu$ is the uniform measure on $[m]$.  Using $\min \mu(S_t) = \frac{1}{m}$ to upper bound the right
side and setting it equal to $\delta$ we obtain
	\begin{align}
	t = \frac{1}{\priveps} \log \frac{m}{\delta}.
	\end{align}
From this we have
	\begin{align}
	\prob\left( z \ge z_{min} + \frac{2}{\priveps} \log \frac{m}{\delta} \right) \le \delta,
	\end{align}
and the result follows.
\end{proof}

\section{Experiments}

In this section we give experimental results for training linear classifiers 
with Algorithms \ref{alg:output} and \ref{alg:objective} on two real datasets.  Imposing privacy requirements necessarily degrades classifier performance.
 Our experiments show that 
provided there is sufficient data, 
objective perturbation (Algorithm \ref{alg:objective})  typically
outperforms the
sensitivity method (\ref{alg:output}) significantly, and achieves error rate close to that of the analogous non-private ERM algorithm.  We first demonstrate how the accuracy 
of the classification algorithms vary with $\priveps$, the privacy requirement. 
We then show how the performance of privacy-preserving classification 
varies with increasing training data size.

The first dataset we consider is the \adult\ dataset from the UCI Machine
Learning Repository \citep{uciadult}. This moderately-sized dataset contains demographic information about approximately $47,000$ individuals, and the
classification task is to predict whether the annual income of an
individual is below or above \$50,000, based on variables such as age, sex,
occupation, and education.  For our experiments, the average fraction of
positive labels is about $0.25$; therefore, a trivial classifier that always
predicts $-1$ will achieve this error-rate, and only error-rates 
below $0.25$ are interesting. 

The second dataset we consider is the \kddcup\ dataset \citep{kddcup99};
the task here is to predict whether a network connection is a denial-of-service
attack or not, based on several attributes. The dataset includes about
5,000,000 instances. For this data the average fraction of positive 
labels is $0.20$. 

In order to implement the convex minimization procedure, we use the convex optimization library provided by~\cite{convexopt}.  

\subsection{Preprocessing}

In order to process the \adult\ dataset into a form amenable for
classification, we removed all entries with missing values, and converted
each categorial attribute to a binary vector.  For example, an attribute such as
\texttt{(Male,Female)} was converted into 2 binary features.  Each column
was normalized to ensure that the maximum value is $1$, and then each row
is normalized to ensure that the norm of any example is at most $1$.  After
preprocessing, each example was represented by a $105$-dimensional vector, of norm at most $1$.

For the \kddcup\ dataset, the instances were preprocessed by converting
each categorial attribute to a binary vector. Each column was normalized to
ensure that the maximum value is $1$, and finally, each row was normalized,
to ensure that the norm of any example is at most $1$. After preprocessing,
each example was represented by a $119$-dimensional vector, of norm at most $1$.

\subsection{Privacy-Accuracy Tradeoff}

For our first set of experiments, we study the tradeoff between the privacy
requirement on the classifier, and its classification accuracy, when the classifier
is trained on data of a fixed size. The privacy requirement is quantified by the value of $\priveps$; increasing $\priveps$ implies a higher change in the belief of the 
adversary when one entry in $\D$ changes, and thus lower privacy. To measure
accuracy, we use classification (test) error; namely, the fraction of times the
classifier predicts a label with the wrong sign.

\begin{figure}[ht]
\centering
\subfigure[Regularized logistic regression, \adult\ ]{
	\includegraphics[scale=0.4]{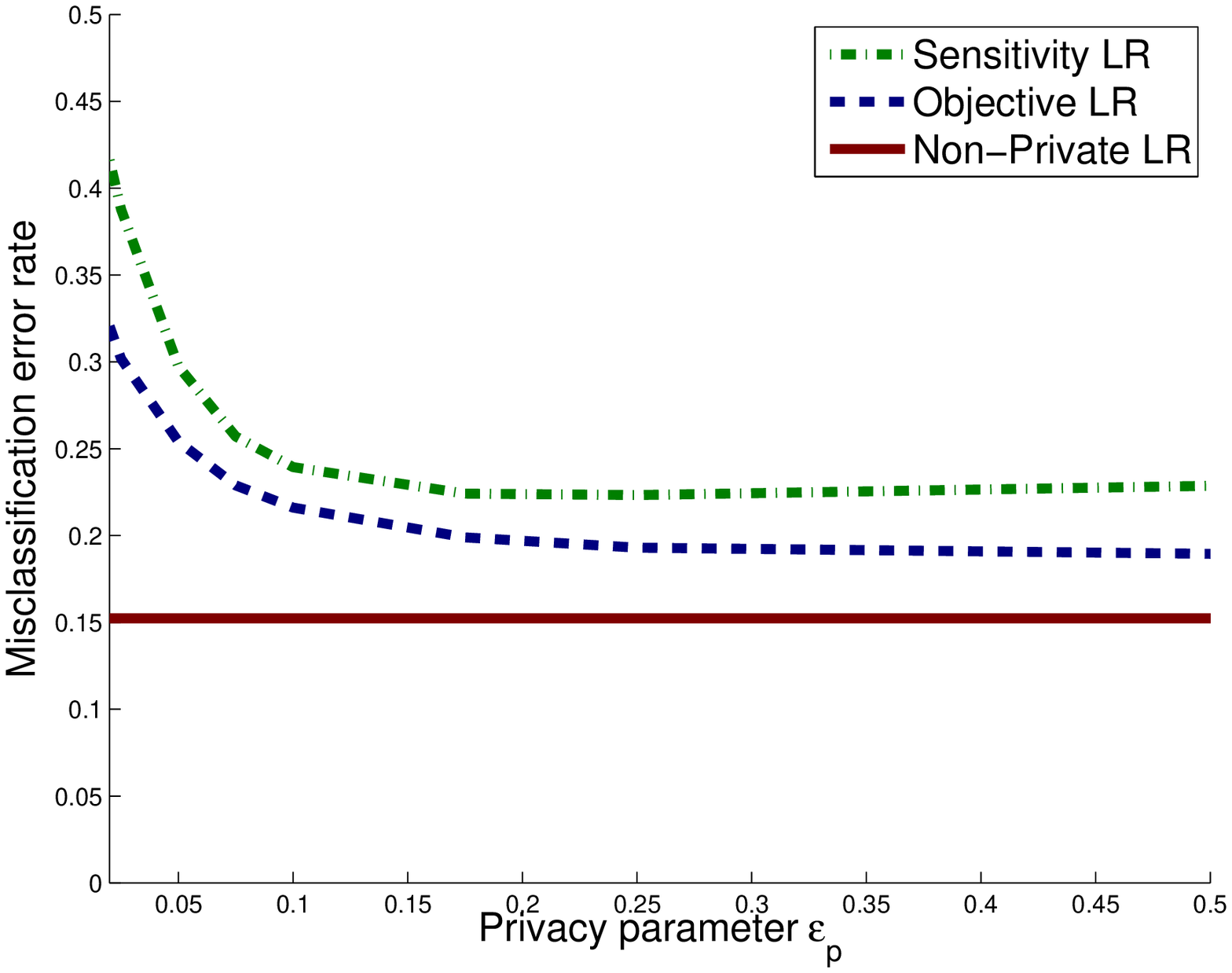}
	}
\subfigure[Regularized SVM, \adult\ ]{
	\includegraphics[scale=0.4]{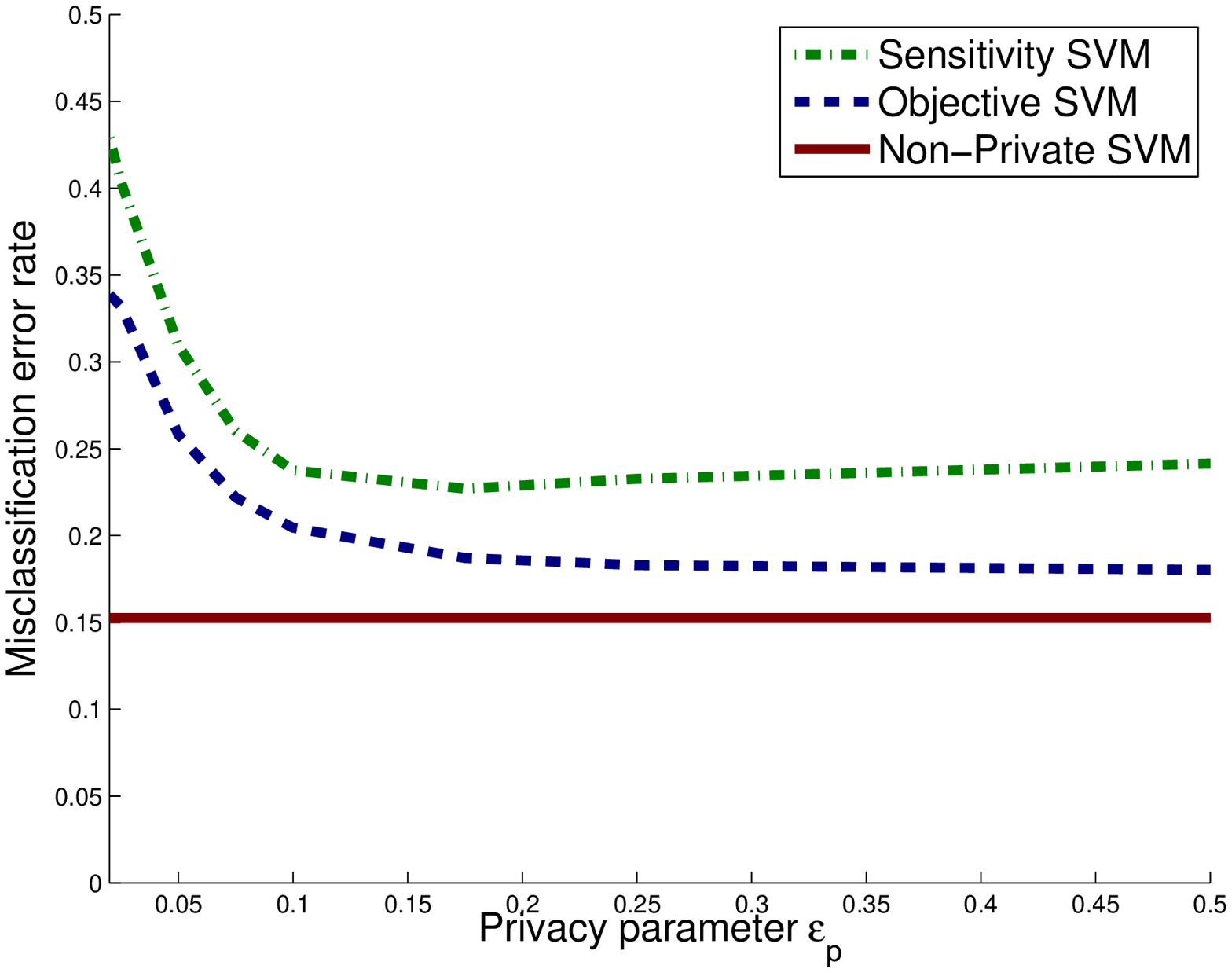}
	}
\caption{Privacy-Accuracy trade-off for the \adult\ dataset}
\label{fig:adulteps}
\end{figure}

\begin{figure}[ht]
\centering
\subfigure[Regularized logistic regression, \kddcup\ ]{
	\includegraphics[scale=0.4]{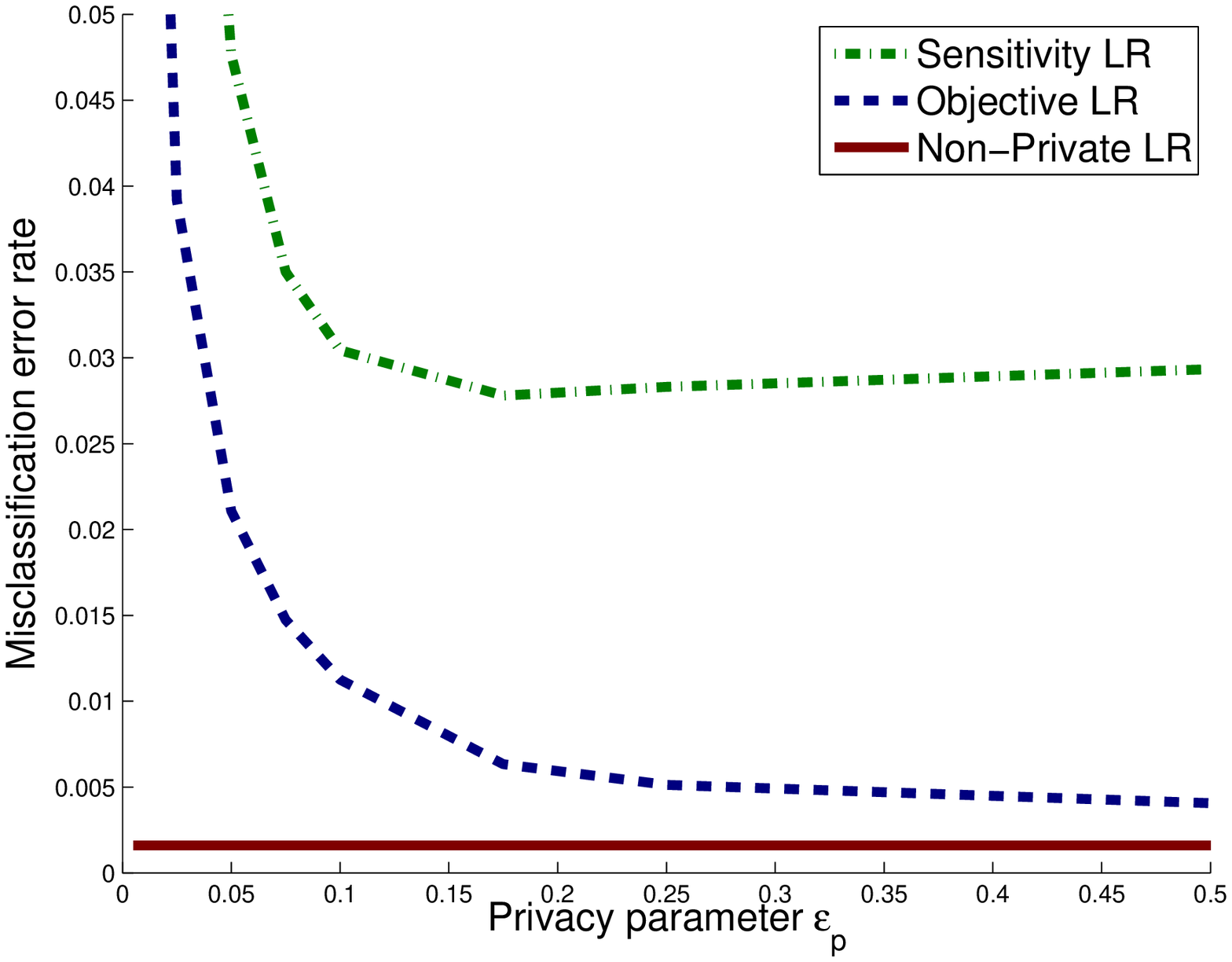}
	}
\subfigure[Regularized SVM, \kddcup\ ]{
	\includegraphics[scale=0.4]{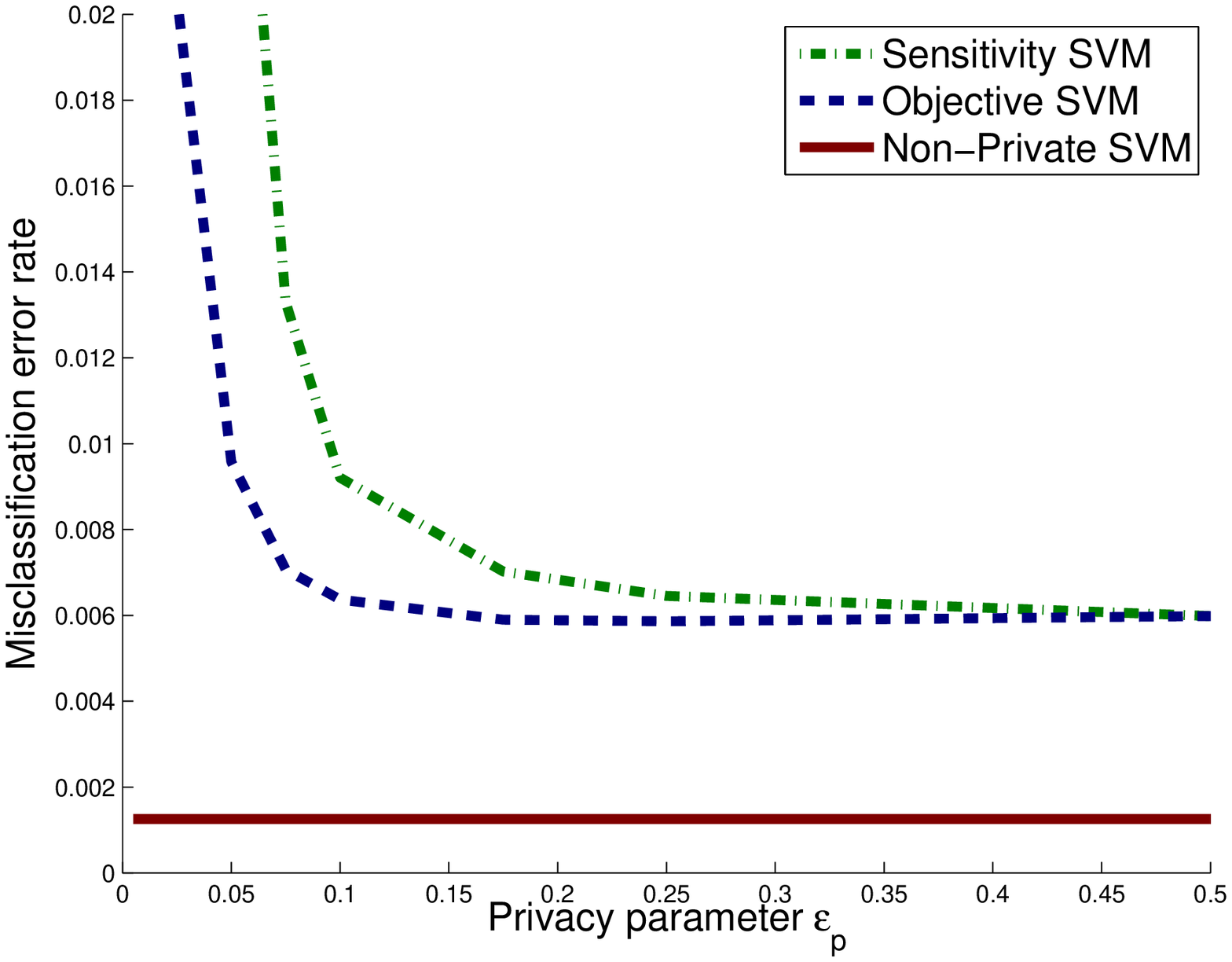}
	}
\caption{Privacy-Accuracy trade-off for the \kddcup\ dataset}
\label{fig:kddeps}
\end{figure}

To study the privacy-accuracy tradeoff, we compare objective perturbation with
the sensitivity method for logistic regression and Huber SVM. For Huber
SVM, we picked the Huber constant $h = 0.5$, a typical value
\citep{Chapelle:07primal}\footnote{\cite{Chapelle:07primal} recommends using
$h$ between $0.01$ and $0.5$; we use $h=0.5$ as we found that a higher value typically
leads to more numerical stability, as well as better performance for both
privacy-preserving methods.}.
For each data set we trained  classifiers for a few fixed values of $\reg$ and tested the error of these classifiers.  For each algorithm %
we chose the value of $\reg$
that minimizes the error-rate for $\priveps = 0.1$.\footnote{For \kddcup\ the error of the non-private algorithms did not increase with decreasing $\reg$.%
}  We then plotted the error-rate against $\priveps$ for the chosen value of $\reg$.  The results are shown in
Figures \ref{fig:adulteps} and \ref{fig:kddeps} for both logistic regression and support vector machines\footnote{The slight kink in the SVM curve on \adult\ is due to a switch to the second phase of the algorithm.}.  The optimal values of $\reg$ are shown in Tables~\ref{tab:regvaluesAdult} and ~\ref{tab:regvaluesKDD}. For non-private logistic regression and SVM, each presented error-rate is an average over $10$-fold cross-validation; for the sensitivity method
as well as objective perturbation, the presented error-rate is an average 
over $10$-fold cross-validation and $50$ runs of the randomized training
procedure.  For \adult, the privacy-accuracy tradeoff is computed over the entire dataset, which consists of $45,220$ examples; for \kddcup\ we use a randomly chosen subset of $70,000$ examples.  

\begin{table}
\centering
\begin{tabular}{c|llllllll}
$\reg$  & $10^{-10.0}$  & $10^{-7.0}$  & $10^{-4.0}$  & $10^{-3.5}$  & $10^{-3.0}$  & $10^{-2.5}$  & $10^{-2.0}$  & $10^{-1.5}$  \\ 
\hline 
\hline
Logistic  & & & & & & & &   \\ 
Non-Private & 0.1540  & \textbf{0.1533}  & 0.1654  & 0.1694  & 0.1758  & 0.1895  & 0.2322  & 0.2478  \\ 
Output & 0.5318  & 0.5318  & 0.5175  & 0.4928  & 0.4310  & 0.3163  & \textbf{0.2395}  & 0.2456  \\ 
Objective & 0.8248  & 0.8248  & 0.8248  & 0.2694  & 0.2369  & \textbf{0.2161}  & 0.2305  & 0.2475  \\ 
\hline
\hline
Huber & & & & & & & &   \\ 
Non-Private & 0.1527  & \textbf{0.1521}  & 0.1632  & 0.1669  & 0.1719  & 0.1793  & 0.2454  & 0.2478  \\ 
Output & 0.5318  & 0.5318  & 0.5211  & 0.5011  & 0.4464  & 0.3352  & \textbf{0.2376}  & 0.2476  \\ 
Objective & 0.2585  & 0.2585  & 0.2585  & 0.2582  & 0.2559  & \textbf{0.2046}  & 0.2319  & 0.2478  \\ 
\end{tabular}
\caption{Error for different regularization parameters on \adult\ for $\priveps = 0.1$.  The best error per algorithm is in bold.}
\label{tab:regvaluesAdult}
\end{table}

\begin{table}
\centering
\begin{tabular}{c|llllllll}
$\reg$  & $10^{-9.0}$  & $10^{-7.0}$  & $10^{-5.0}$  & $10^{-3.5}$  & $10^{-3.0}$  & $10^{-2.5}$  & $10^{-2.0}$  & $10^{-1.5}$ \\
\hline
\hline 
Logistic & & & & & & & &  \\ 
Non-Private & 0.0016  & \textbf{0.0016}  & 0.0021  & 0.0038  & 0.0037  & 0.0037  & 0.0325  & 0.0594  \\ 
Output & 0.5245  & 0.5245  & 0.5093  & 0.3518  & 0.1114  & 0.0359  & \textbf{0.0304}  & 0.0678  \\ 
Objective & 0.2084  & 0.2084  & 0.2084  & 0.0196  & 0.0118  & \textbf{0.0113}  & 0.0285  & 0.0591  \\ 
\hline
\hline
Huber & & & & & & & &  \\ 
Non-Private & \textbf{0.0013}  & 0.0013  & 0.0013  & 0.0029  & 0.0051  & 0.0056  & 0.0061  & 0.0163  \\ 
Output & 0.5245  & 0.5245  & 0.5229  & 0.4611  & 0.3353  & 0.0590  & \textbf{0.0092}  & 0.0179  \\ 
Objective & 0.0191  & 0.0191  & 0.0191  & 0.1827  & 0.0123  & 0.0066  & \textbf{0.0064}  & 0.0157  \\ 
\end{tabular}
\caption{Error for different regularization parameters on \kddcup\ for $\priveps = 0.1$.  The best error per algorithm is in bold. %
}
\label{tab:regvaluesKDD}
\end{table}

For the \adult\ dataset, the constant classifier that classifies all
examples to be negative acheives a classification error of about $0.25$.
The sensitivity method thus does slightly better than this constant
classifier for most values of $\priveps$ for both logistic regression and
support vector machines. Objective perturbation outperforms sensitivity,
and objective perturbation for support vector machines achieves lower
classification error than objective perturbation for logistic regression.
Non-private logistic regression and support vector machines both have
classification error about $0.15$. 

For the \kddcup\ dataset, the constant classifier that classifies all
examples as negative, has error $0.19$. Again, objective perturbation
outperforms sensitivity for both logistic regression and support vector
machines; however, for SVM and high values of $\priveps$ (low privacy), the
sensitivity method performs almost as well as objective perturbation.  In
the low privacy regime, logistic regression under objective perturbation is
better than support vector machines.  In contrast, in the high privacy
regime (low $\priveps$), support vector machines with objective perturbation
outperform logistic regression. For this dataset, non-private logistic regression and support vector machines both have a classification error of about $0.001$.

For SVMs on both \adult\ and \kddcup, for large $\priveps$ ($0.25$
onwards), the error of either of the private methods can increase slightly with increasing $\priveps$. This
seems  counterintuitive, but appears to be due the imbalance in
fraction of the two labels. As the labels are imbalanced, the optimal
classifier is trained to perform better on the negative labels than the
positives.  As $\priveps$ increases, for a fixed training data size, so
does the perturbation from the optimal classifier, induced by either of the private methods. Thus, as the
perturbation increases, the number of false positives increases, whereas
the number of false negatives decreases (as we verified by 
measuring the average false
positive and false negative rates of the private classifiers).
Therefore, the total error may increase slightly with decreasing privacy.

\subsection{Accuracy vs. Training Data Size Tradeoffs}

Next we examine how classification accuracy varies as we increase the size of the training set. We measure classification accuracy as the accuracy of the classifier produced by the tuning procedure in Section~\ref{sec:tuning}. As the \adult\ dataset is not sufficiently large to allow us to do privacy-preserving tuning, for these experiments, we
restrict our attention to the \kddcup\ dataset.

Figures \ref{fig:lclr} and \ref{fig:lcsvm} present the learning curves for objective
perturbation, non-private ERM and the sensitivity method for logistic loss and Huber loss, respectively.  Experiments are shown for $\priveps=0.01$ and $\priveps=0.05$ for both loss functions.  The training sets (for each  of 5 values of $\reg$) are chosen to be
of size $n=60,000$ to $n=120,000$, and the validation and test sets each
are of size $25,000$.   Each presented value is an average over $5$ random
permutations of the data, and $50$ runs of the randomized classification
procedure.  For objective perturbation we performed experiment in the
regime when $\priveps' > 0$, so $\Delta = 0$ in Algorithm
\ref{alg:objective}.\footnote{This was chosen for a fair comparison with non-private as well as the output perturbation method, both of which had access to only $5$ values of
$\reg$.}

For non-private ERM, we present result for training sets from $n=300,000$ to $n=600,000$.  
The non-private algorithms are tuned by comparing $5$ values of $\reg$ on the same training set, and the test set is of size $25,000$. Each reported value is an average over $5$ random permutations of the data.

We see from the figures that for non-private logistic regression and
support vector machines, the error remains constant with increasing data
size. For the private methods, the error usually decreases as the data size
increases. In all cases, objective perturbation outperforms the sensitivity
method, and support vector machines generally outperform  logistic regression.

\begin{figure}[ht]
\centering
\subfigure[$\priveps = 0.05$]{
	\includegraphics[scale=0.4]{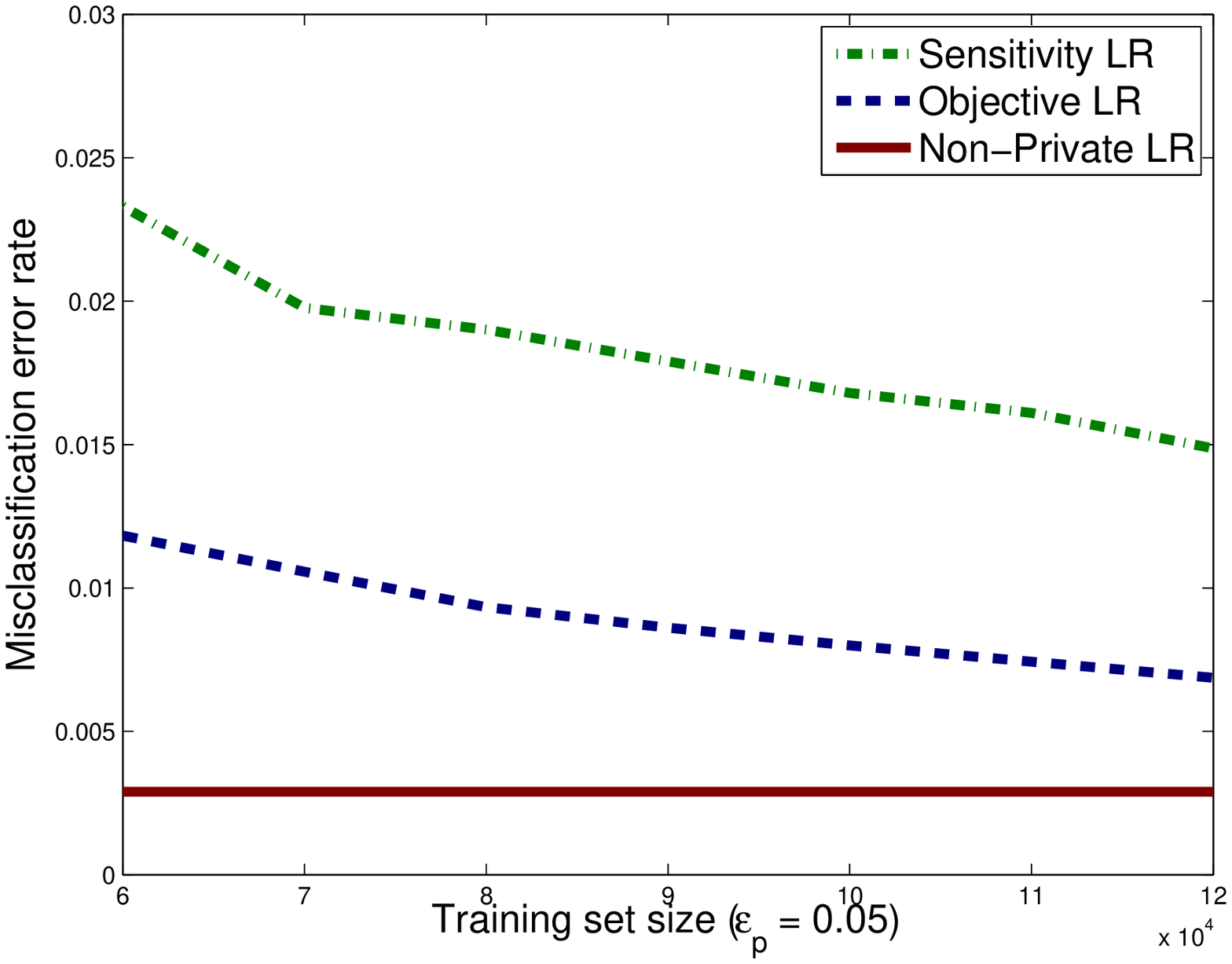}
	}
\subfigure[$\priveps = 0.01$]{
	\includegraphics[scale=0.4]{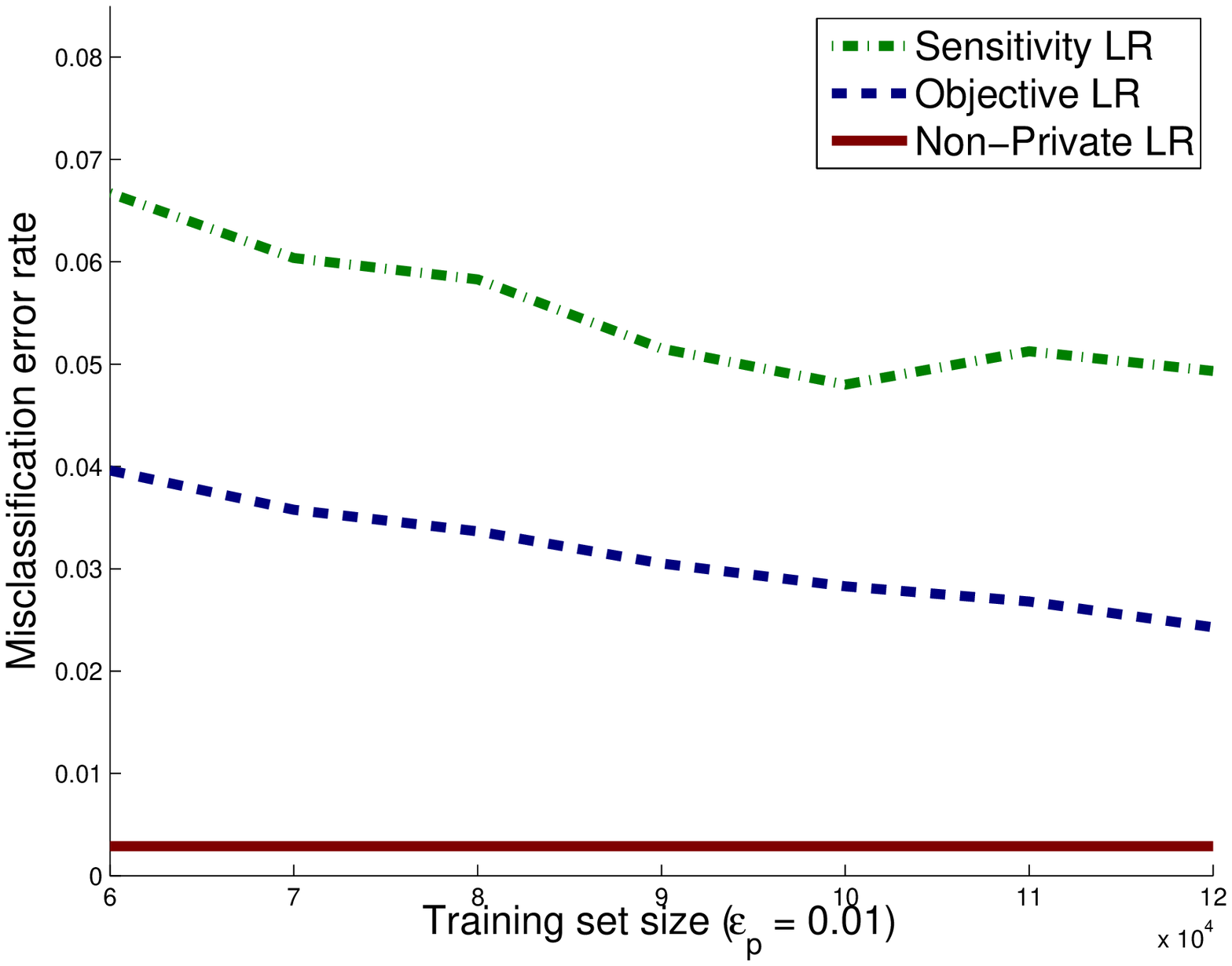}
	}
\caption{Learning curves for logistic regression on the \kddcup\ dataset}
\label{fig:lclr}
\end{figure}

\begin{figure}[ht]
\centering
\subfigure[$\priveps = 0.05$]{
	\includegraphics[scale=0.4]{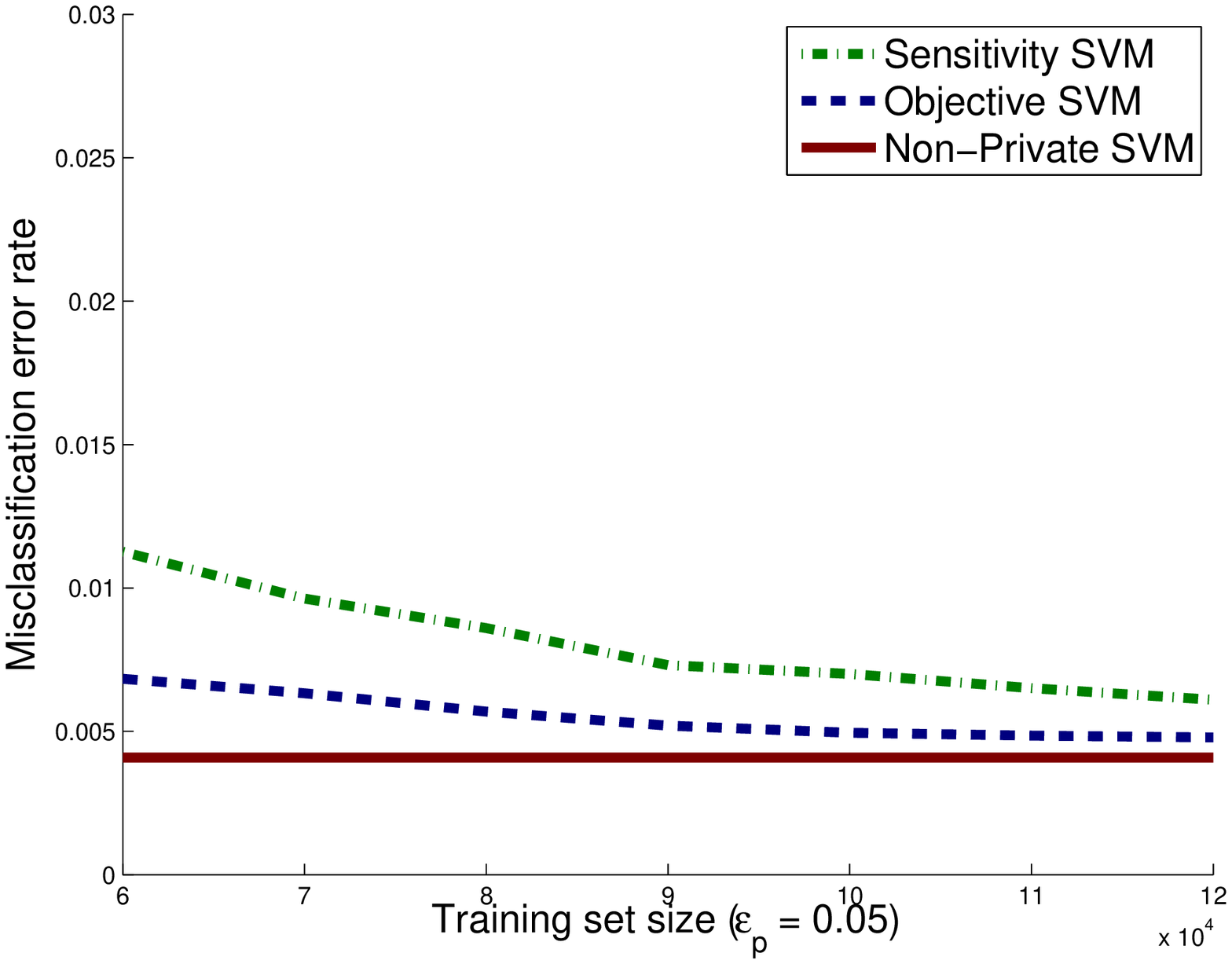}
	}
\subfigure[$\priveps = 0.01$]{
	\includegraphics[scale=0.4]{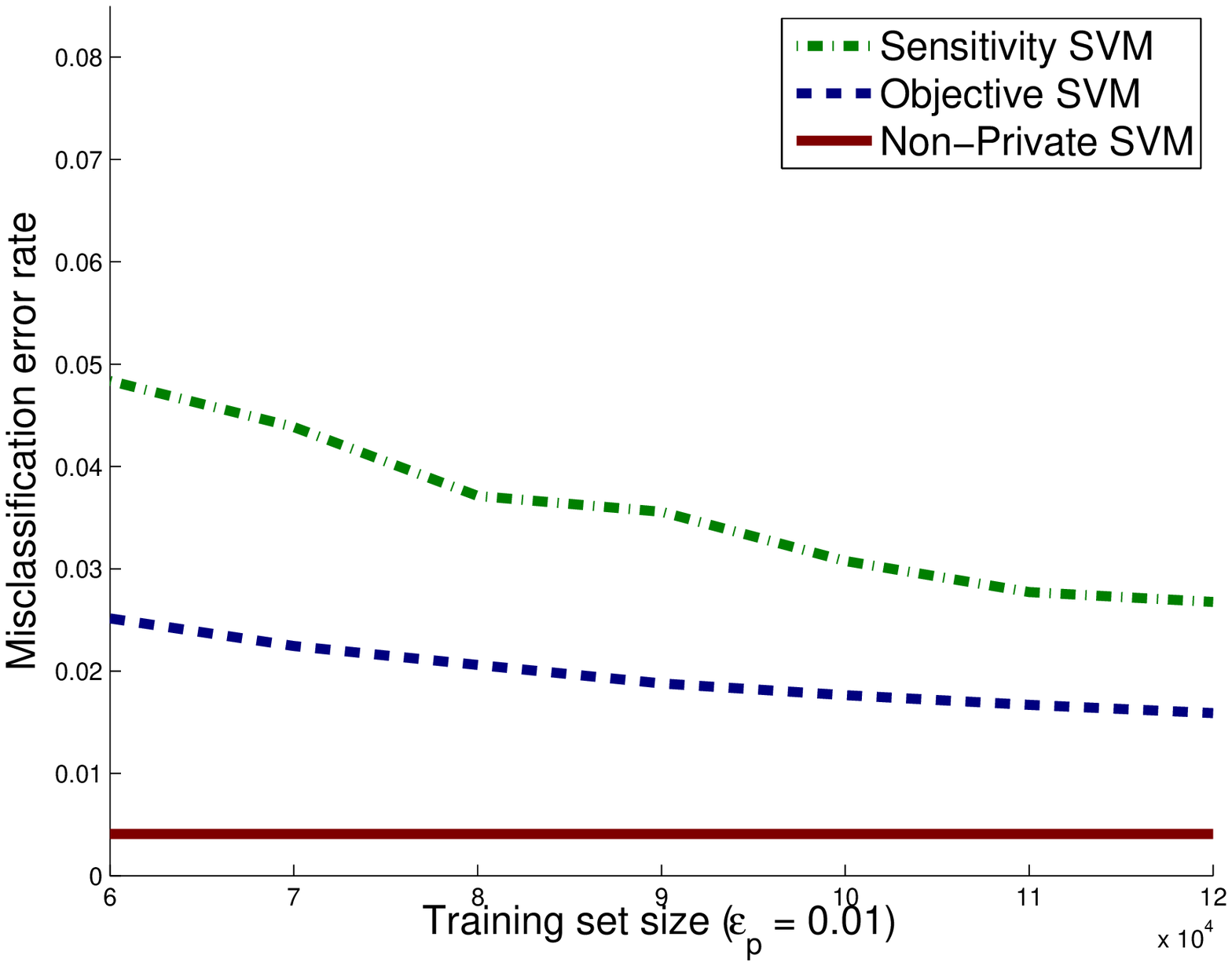}
	}
\caption{Learning curves for SVM on the \kddcup\ dataset }
\label{fig:lcsvm}
\end{figure}

\section{Discussions and Conclusions}

In this paper we study the problem of learning classifiers with
regularized empirical risk minimization in a privacy-preserving manner.
We consider privacy in the $\priveps$-differential privacy model
of~\cite{DworkMNS:06sensitivity} and provide two algorithms for
privacy-preserving ERM.   The first one is based on the
sensitivity method due to~\cite{DworkMNS:06sensitivity}, in which the
output of the non-private algorithm is perturbed by adding noise.  
We introduce a second algorithm based on the new paradigm of objective perturbation. We
provide bounds on the sample requirement of these algorithms for
achieving generalization error $\geneps$. We show how to apply these
algorithms with kernels, and finally, we provide experiments with both
algorithms on two real datasets.   Our work is, to our knowledge, the first to propose computationally efficient classification algorithms satisfying differential privacy, together with validation on standard data sets.

In general, for classification, the error rate increases as the
privacy requirements are made more stringent.  Our generalization
guarantees formalize this ``price of privacy.'' Our experiments, as
well as theoretical results, indicate that objective perturbation
usually outperforms the sensitivity methods at managing the tradeoff
between privacy and learning performance.  Both algorithms perform
better with more training data, and when abundant training data is
available, the performance of both algorithms can be close to
non-private classification. 

The conditions on the loss function and regularizer 
required by output perturbation and objective perturbation
 are somewhat different. As Theorem~\ref{thm:sensitivitypriv} shows, output 
 perturbation requires strong convexity in
 the regularizer and convexity as well as a bounded derivative condition in the loss
 function. The last condition can be replaced by a Lipschitz condition
 instead. However, the other two conditions appear to be required, unless
 we impose some further restrictions on the loss and regularizer. Objective
 perturbation on the other hand, requires strong convexity of the
 regularizer, convexity, differentiability, and bounded double derivatives in
 the loss function. Sometimes, it is possible to construct a
 differentiable approximation to the loss function, even if the loss
 function is not itself differentiable, as
 shown in Section~\ref{sec:svmprivacy}.  
 
 Our experimental as well as theoretical results indicate that in general,
 objective perturbation provides more accurate solutions than output
 perturbation. Thus, if the loss function satisfies the conditions of
 Theorem~\ref{thm:objectivepriv}, we recommend using objective
 perturbation. In some situations, such as for SVMs, it is possible that
 objective perturbation does not directly apply, but applies to an
 approximation of the target loss function. In our experiments, the
 loss of statistical efficiency due to such approximation has been small compared to the loss
 of efficiency due to privacy, and we suspect that this is the case for many
 practical situations as well. 

Finally, our work does not address the question of finding private solutions to
regularized ERM when the regularizer is not
strongly convex. For example, neither the output perturbation, nor the
objective perturbation method work for $L_1$-regularized ERM.
However, in $L_1$-regularized ERM, one can find a dataset in which
a change in one training point can significantly change the solution. 
As a result, it is possible that such problems are 
inherently difficult to solve privately. 

An open question in this work is to extend objective perturbation
methods to more general convex optimization problems. Currently, the
objective perturbation method applies to strongly convex regularization 
functions and differentiable losses.  Convex optimization problems appear
in many contexts within and without machine learning: density estimation,
resource allocation for communication systems and networking, social 
welfare optimization in economics, and elsewhere.  In some cases these algorithms
will also operate on sensitive or private data.  Extending the ideas and
analysis here to those settings would provide a rigorous foundation for
privacy analysis.

A second open question is to find a better solution for privacy-preserving
classification with kernels. Our current method is based on a reduction to
the linear case, using the algorithm of \cite{RahimiR:08kitchen}; however,
this method can be statistically inefficient, and require a lot of training data, 
particularly when coupled with our privacy mechanism.   The reason is that the algorithm of \cite{RahimiR:08kitchen} requires the dimension $D$ of the projected space
to be very high for good performance.  However, most differentially-private
algorithms perform worse as the dimensionality of the data grows. Is there a better
linearization method, which is possibly data-dependent, that will provide a more
statistically efficient solution to privacy-preserving learning with
kernels?

A final question is to provide better upper and lower bounds on
the sample requirement of privacy-preserving linear classification. 
The main open question here is to provide a computationally efficient algorithm for 
linear classification which has better statistical efficiency.

Privacy-preserving machine learning is the endeavor of designing private analogues of widely used machine learning algorithms. 
We believe the present study is a starting point for further study of the differential privacy model in this relatively new subfield of machine learning. The work of \cite{DworkMNS:06sensitivity} set up a framework for assessing the privacy risks associated with publishing the results of data analyses.  Demanding high privacy requires sacrificing utility, which in the context of classification and prediction is excess loss or regret.  In this paper we demonstrate the privacy-utility tradeoff for ERM, which is but one corner of the machine learning world.  Applying these privacy concepts %
to other machine learning problems will lead to new and interesting tradeoffs and towards a set of tools for practical privacy-preserving learning and inference.  
We hope that our work provides a benchmark of the current
price of privacy, and inspires improvements in future work.

\subsection*{Acknowledgments}

The authors would like to thank Sanjoy Dasgupta and Daniel Hsu for
several pointers, and to acknowledge Adam Smith, Dan Kifer, and Abhradeep Guha Thakurta, who helped point out an error in the previous version of the paper.   The work of K. Chaudhuri and A.D. Sarwate was supported in part by the California Institute for Telecommunications and Information Technologies (CALIT2) at UC San Diego.  K. Chaudhuri was also supported by National Science Foundation IIS-0713540.   Part of this work was done while C. Monteleoni was at UC San Diego, with support from National Science Foundation IIS-0713540.   The experimental results were made possibly by support from the UCSD FWGrid Project, NSF Research Infrastructure Grant Number EIA-0303622.

\vskip 0.2in
\bibliography{privacy}

\end{document}